\documentclass{article}

\usepackage{amsmath,amsfonts,bm}

\def\eqref#1{equation~\ref{#1}}
\def\Eqref#1{Equation~\ref{#1}}

\def\1{\bm{1}}

\def\rvx{{\mathbf{x}}}

\def\vx{{\bm{x}}}

\DeclareMathAlphabet{\mathsfit}{\encodingdefault}{\sfdefault}{m}{sl}
\SetMathAlphabet{\mathsfit}{bold}{\encodingdefault}{\sfdefault}{bx}{n}

\def\gA{{\mathcal{A}}}

\def\gD{{\mathcal{D}}}

\def\gF{{\mathcal{F}}}

\def\gX{{\mathcal{X}}}
\def\gY{{\mathcal{Y}}}

\usepackage{microtype}
\usepackage{graphicx}
\usepackage{booktabs} 
\usepackage{xcolor}      
\usepackage{algorithm}
\usepackage{algorithmic}

\usepackage{amsfonts}
\usepackage{multirow}
\usepackage{amsmath}
\usepackage{graphicx}
\usepackage{amssymb}
\usepackage{subcaption}
\usepackage{amsthm}
\usepackage{natbib}
\usepackage{afterpage}

\newtheorem{thm}{Theorem}[section]
\newtheorem{cor}[thm]{Corollary}
\newtheorem{prop}[thm]{Proposition}
\newtheorem{lem}[thm]{Lemma}
\newtheorem{definition2}[thm]{Definition}

\usepackage{tikz}
\usetikzlibrary{decorations.pathmorphing}

\newcommand{\ma}{\mathcal{A}}
\newcommand{\mx}{\mathcal{X}}
\newcommand{\my}{\mathcal{Y}}

\newcommand{\mb}{\mathcal{B}}
\newcommand{\mc}{\mathcal{C}}
\newcommand{\mt}{\mathcal{T}}
\newcommand{\md}{\mathcal{D}}

\newcommand{\mi}{\mathcal{I}}

\newcommand{\ignore}[1]{}

\newcommand{\prob}[2]{\mathbb{P}_{#1}\left[#2\right]}

\newcommand\numberthis{\addtocounter{equation}{1}\tag{\theequation}}
\newcommand{\topdown}{TopDown}

\usepackage{hyperref}

\usepackage[accepted]{icml2019}

\icmltitlerunning{On the Optimality of Trees Generated by ID3}

\begin{document}

\twocolumn[
\icmltitle{On the Optimality of Trees Generated by ID3}



\icmlsetsymbol{equal}{*}

\begin{icmlauthorlist}
\icmlauthor{Alon Brutzkus}{tau}
\icmlauthor{Amit Daniely}{huji}
\icmlauthor{Eran Malach}{huji}

\end{icmlauthorlist}

\icmlaffiliation{tau}{The Blavatnik School of Computer Science,	Tel Aviv University, Israel.}
\icmlaffiliation{huji}{School of Computer Science, The Hebrew University, Israel.}

\icmlcorrespondingauthor{Alon Brutzkus}{alonbrutzkus@mail.tau.ac.il}

\icmlkeywords{Machine Learning, ICML}

\vskip 0.3in
]



\printAffiliationsAndNotice{}  

\begin{abstract}
	Since its inception in the 1980s, ID3 has become one of the most successful and widely used algorithms for learning decision trees. However, its theoretical properties remain poorly understood. In this work, we introduce a novel metric of a decision tree algorithm's performance, called mean iteration statistical consistency (MIC), which measures optimality of trees generated by ID3. As opposed to previous metrics, MIC can differentiate between different decision tree algorithms and compare their performance. We provide theoretical and empirical evidence that the TopDown variant of ID3, introduced by Kearns and Mansour (1996), has near-optimal MIC in various settings for learning read-once DNFs under product distributions. In contrast, another widely used variant of ID3 has MIC which is not near-optimal. We show that the MIC analysis predicts well the performance of these algorithms in practice. Our results present a novel view of decision tree algorithms which may lead to better and more practical guarantees for these algorithms.
	
\end{abstract}

\section{Introduction}

Decision tree algorithms are widely used in various learning tasks and competitions. The most popular algorithms, which include ID3 \citep{quinlan1986induction} and its successors C4.5 and CART, use a greedy top-down approach to grow trees. In each iteration, ID3 chooses a leaf and replaces it with an internal node connected to two new leaves. This splitting operation is based on a splitting criterion, which promotes reduction of the training error. The popularity of this algorithm stems from its simplicity, interpretability and good generalization performance.

Despite its success in practice, the theoretical properties of ID3 are not well understood. The main drawback of current theoretical analyses is that they rarely predict performance in practice and do not enable comparison between different algorithms. This is because they either hold for very large trees which are not used in practice \citep{kearns1996boosting}, or rely on generalization upper bounds from which it is difficult to differentiate between algorithms' performance \citep{mcallester1999some}. 


In this work, we analyze the ID3 algorithm with a novel metric which we call mean iteration statistical consistency (MIC). MIC  measures the difference between the test error of trees generated by ID3 to the test error of the optimal tree of the same size, across different iterations. We analyze the MIC of the TopDown algorithm \citep{kearns1999boosting}, a variant of ID3. TopDown is an implementation of ID3, where in each iteration the leaf that it chooses to split is the one with the largest gain reduction weighted by the probability to reach the leaf. We provide theoretical and empirical evidence, that TopDown has near-optimal MIC. 

On the theory side, we prove that TopDown has optimal MIC in two settings. We show this for learning conjunctions under product distributions and learning read-once DNFs with two terms under the uniform distribution. We devise novel dynamic programming algorithms to calculate trees with optimal test error for a large number of settings of learning read-once DNFs under product distribitions. We use this to compute the MIC of TopDown and another variant of ID3, which we call BestFirst \citep{shi2007best} and is widely used in practice. We empirically show that TopDown has near-optimal MIC, while BestFirst does not have near-optimal MIC. We corrobarate our findings and show empirically that TopDown has significantly better test performance than BestFirst in our setting.

Our results show that MIC is a useful measure for comparing decision tree algorithms and predicting their performance in practice. Furthermore, it measures how far a given algorithm is from the best possible decision tree algorithm for a given setting. This is beneficial for understanding if an algorithm can be improved. This demonstrates the merits of the MIC analysis and corroborates the success of ID3 algorithms in practice. We believe that our results can pave the way for better understanding decision tree algorithms and lead to meaningful practical insights.

\section{Related Work}

The ID3 algorithm was introduced by \citet{quinlan1986induction}. There are a few papers which study its theoretical properties. The main difference between our work and previous ones, is that we provide test guarantees for trees of practical size by analyzing trees generated by ID3 in each iteration. In contrast, previous works provide guarantees for ID3 in the cases of building a large tree which implements the target function exactly or a tree with a very large polynomial size.

The work which is most related to ours, is \cite{fiat2004decision}, which show that ID3 can learn read-once DNF and linear functions under the uniform distribution in polynomial time. They also show that ID3 builds the tree with minimal size among all trees that implement the ground-truth function. Their results require ID3 to build a tree which implements the ground-truth exactly. This can result in a very large tree in practice.

Recently, \citet{brutzkus2019id3} use smoothed analysis to show that ID3 can learn $\log(n)$-juntas over $n$ variables under product distributions in polynomial time. Their result requires to build a tree with a large polynomial size that implements the target function exactly. Another related work is \cite{kearns1999boosting}, which introduce the TopDown variant of ID3. They show that TopDown is a boosting algorithm under the assumptions that there is a weak approximation of the target function in each node. To get a test error guarantee of $\epsilon$, there result requires to build a tree with at least $\left(\frac{1}{\epsilon}\right)^{128}$ nodes, which is not practical. 

Other works study learnability of decision trees through algorithms which are different from algorithms used in practice \citep{o2007learning,kalai2008decision,bshouty2003proper,bshouty2005learning,ehrenfeucht1989learning,chen2018beyond} or show hardness results for learning decision trees \citep{rivest1987learning, hancock1996lower}. The statistical consitency of random forest algorithms has been studied in several works \cite{wager2014asymptotic, denil2014narrowing, biau2008consistency, biau2012analysis}.

\section{Preliminaries}
\label{sec:preliminaries}

\underline{\textbf{Distributional Assumptions:}} Let $\mathcal{X} = \{0,1\}^n$ be the domain and $\mathcal{Y} = \{0,1\}$ be the label set. Let $\mathcal{D}$ be a product distribution on $\mathcal{X} \times \mathcal{Y}$ realizable by a read-once DNF. Namely, for $(\rvx,y) \sim \mathcal{D}$ it holds that $\rvx \sim \prod_{i=1}^n{Bernoulli(p_i)}$ where $p_1,...,p_n \in (0,1)$ and $y=f(\vx)$ for a read-once DNF $f:\mx \rightarrow \my$. Recall that a read-once DNF is a DNF where each variable appears at most once, e.g., $f(\rvx) = \left(x_1 \wedge x_2 \wedge \bar{x_3} \right) \vee \left(\bar{x_4} \wedge x_5\right)$.

\underline{\textbf{Decision Trees:}} Let $T$ be any decision tree whose internal nodes are labeled with features $\{x_i\}_{i=1}^n$. For a node in the tree $v$, we let $p_T(v)$ be the probability that a randomly chosen $\rvx$ reaches $v$ in $T$ and let $q_T(v)$ be the probability that $f(\rvx) = 1$ given that $\rvx$ reaches $v$. For convenience, we will usually omit the subscript $T$ from the latter definitions when the tree used is clear from the context. Let $\ell(T)$ be the set of leaves of $T$ and $\mi(T)$ be the set of internal nodes (non-leaves). We assume that each leaf is labeled $1$ if $q(l) \ge \frac{1}{2}$ and $0$ otherwise. If $l \in \ell(T)$ we let $T(l,i)$ be the same as the tree $T$, except that the leaf $l$ is replaced with an internal node labeled by $x_i$ and connected to two leaves $l_0$ and $l_1$. The leaf $l_j$ corresponds to the assignment $x_i = j$ and each leaf is labeled according to the majority label with respect to $\md$ conditioned on reaching the leaf. For a leaf $l \in \ell(T)$, let $F_l$ be the set of features that are not on the path from the root to $l$. Finally, we let $\mt_t$ be the set of all decision trees with at most $t$ internal nodes. 

Let $E(T) = \prob{\md}{T(\rvx) \neq f(\rvx)}$ be the error of the tree $T$. Then it holds that,
$E(T) = \sum_{l\in\ell(T)}{p(l)C(q(l))}$
where $C(x) = \min(x,1-x)$. Let $H(x) = -x\log(x)-(1-x)\log(x)$ be the entropy function, where the $\log$ is base $2$, and define the entropy of $T$ to be $H(T) = \sum_{l\in\ell(T)}{p(l)H(q(l))}$
which satisfies $E(T) \le H(T)$.

We will assume that the algorithms we study can exactly compute probabilities with respect to the distribtion $\mathcal{D}$. For example, they can compute $p(l)$ or $q(l)$ for any leaf. This will be convenient for our analysis of the MIC measure.\footnote{Using finite sample concentration inequalities the probabilities can be computed to a desired accuracy with polynomial-size samples as in \citep{fiat2004decision}.} 

\underline{\textbf{Algorithm:}} 
The ID3 algorithm \citep{quinlan1986induction} generates decision trees in a recursive manner. In each recursive step, it chooses a variable to split a given leaf. The variable that is chosen is the one with the highest information gain. Then, it continues to split the new leaves in a recursive manner.

The \topdown\, algorithm introduced by  \citet{kearns1999boosting} is a variant of ID3. 
The main difference between \topdown\,and ID3 is the choice of the splitting node in each iteration. Instead of recursively growing the tree, in each iteration \topdown\, chooses the leaf which maximally decreases $H(T)$ and therefore hopefully reduces $E(T)$ as well. Formally, in each iteration, it chooses a leaf $l$ and feature $x_i$, where $i \in F_l$, which maximize:

\begin{align*}
\label{eq:weighted_gain}
H(T) - H(T(l,i)) &= p(l)\Big(H(q(l)) \\ &-(1-\tau_i) H(q(l_0)) - \tau_i H(q(l_1))\Big) \numberthis
\end{align*}
where $\tau_i$ is the probability that $x_i = 1$ given that $\rvx$ reaches $l$.
We let $T_t$ be the tree computed by \topdown\, at iteration $t$. The algorithm is given in Figure \ref{fig:topdown}.

\begin{figure}
	\begin{algorithm}[H]
		\caption{$\text{TopDown}_{\md}(t)$}
		\label{alg:id3}
		\begin{algorithmic}
			\STATE Initialize $T$ to be a single leaf labeled by the majority label with respect to $\md$.
			\STATE \textbf{while} $T$ has less than $t$ nodes: 
			\begin{ALC@g}
				\STATE $\Delta_{best} \leftarrow 0$.
				\STATE \textbf{for} each pair $l \in \ell(T)$ and $i \in F_l$:
				\begin{ALC@g}
					\STATE $\Delta \leftarrow H(T) - H(T(l,i))$.
					\STATE \textbf{if} $\Delta \ge \Delta_{best}$ \textbf{then}:
					\begin{ALC@g}
					\STATE $\Delta_{best} \leftarrow \Delta$; $l_{best} \leftarrow l$; $i_{best} \leftarrow i$.
					\end{ALC@g}
				\end{ALC@g}
				\STATE $T \leftarrow T(l_{best},i_{best})$.
			\end{ALC@g}
		\STATE \textbf{return} $T$
		\end{algorithmic}
	\end{algorithm}
	\caption{TopDown algorithm.}
	\label{fig:topdown}
\end{figure}

\section{Mean Iteration Statistical Consistency}
\label{sec:msi}

In this section we define the mean iteration statistical consistency metric (MIC). We consider deterministic decision tree algorithm $\gA$, which given access to a distribution $\gD$ and number $t> 0$, \footnote{Formally, we assume it can compute exact probabilites with respect to the distribution.} returns a tree $\gA(\gD,t)$ with at most $t$ internal nodes, i.e., $\gA(\gD,t) \in \mt_t$. We refer to $t$ as the number of iterations of the algorithm $\gA$.  
We assume that the algorithm is efficient, i.e., runs in polynomial time.
 
For any distrubution $\gD$, with ground-truth DNF $f$ and $t>0$, let $OPT(\gD,t)$ be the minimal test error among all trees with at most $t$ internal nodes, with respect to $\gD$. Let $E(\gA(\gD,t))$ be the test error of $\gA(\gD,t)$ with respect to $\gD$.

\begin{definition2}
	For each $t > 0$ define: $$\epsilon_t = E(\gA(\gD,t)) - OPT(\gD,t)$$ We call the constants $\{\epsilon_t\}_{t > 0}$ the set of MIC constants. 
\end{definition2}

The constant $\epsilon_t$ indicates how much the algorithm $\gA$ is far from optimal after $t$ iterations. Ideally, a good algorithm should have $\epsilon_t$ which is close to 0 for any $t$. To measure this, we introduce the following notion. 

\begin{definition2}
	Let $\{\epsilon_t\}_{t > 0}$ be the set of MIC constants of an algorithm $\gA$ with respect to $\gD$. The mean iteration statistical consistency metric (MIC) of an algorithm $\gA$ with respect to $\gD$, is the function $m_{\gD}^{\gA}$, such that: 
	\begin{equation}
	m_{\gD}^{\gA}(t^*) = \frac{1}{t^*}\sum_{0 < t \le t^*}\epsilon_t
	\end{equation}
	for $t^* > 0$.
\end{definition2}	

The MIC measures the performance of an algorithm $\gA$ on trees of practical size, for suitable $t^*$, and compares it to the best possible trees it can output.

The following theorem shows that obtaining the optimal tree is NP-Hard. The proof is given in the supplementary.

\begin{thm}
	\label{thm:nph}
	 The problem of finding the tree of size $t$ with test error $OPT(\gD,t)$, given access to $\gD$, is NP-Hard. \footnote{We assume that there is an oracle that can compute exact probabilities with respect to $\gD$ for any set in $\gX \times \gY$.}
\end{thm} 


Given this result, a desirable condition for a good learning algorithm is that $m_{\gD}^{\gA}(t^*) \approx 0$ for $t^*$ of practical size. Indeed, in practice the tree size is constrained to avoid overfitting and we should expect the algorithm to have good performance for sufficiently large training sets. The performance in this case is measured by the MIC.
 
MIC has two important properties. First, it allows to compare between algorithms by comparing their MIC. We will use this to compare TopDown with another variant of ID3, BestFirst. Second, if we find an algorithm $\gA$ such that $m_{\gD}^{\gA}(t^*) \approx 0$, then any other decision algorithm will not have significantly better performance. This may save efforts for attempts on improving a given algorithm. We will show that TopDown has near-optimal MIC  and therefore any other decision tree algorithm does not have significantly better performance than TopDown in the  setting we consider.

\ignore{
The MIC measure has two important properties with potential practical applications:
\begin{enumerate}
	\item \textbf{Comparison of algorithms on trees of practical size}. Assume there are two algorithms $\gA_1$ and $\gA_2$ such that for a sufficiently large $t^*$ and a set of distributions $\left\{\gD_i\right\}_i$, $m_{\gD_i}^{\gA_1}(t^*) << m_{\gD_i}^{\gA_2}(t^*)$ for all $i$. Then for a sufficiently large training set, algorithm $\gA_1$ has better performance than $\gA_2$ with trees of practical size. This suggests that algorithm $\gA_1$ should be used over $\gA_2$ for learning distributions in $\left\{\gD_i\right\}_i$.
	\item \textbf{Algorithm that has close to optimal performance should be the method of choice.}
Assume that for all $i$, an algorithm $\gA$ satisfies $m_{\gD_i}^{\gA}(t^*) \approx 0$,  where $t^*$ is sufficiently large and $\left\{\gD_i\right\}_i$ is a set of distributions. Since the problem of finding optimal trees is NP-Hard, this suggests that $\gA$ is a good choice for learning distirbutions in $\left\{\gD_i\right\}_i$ with decision trees and is unlikely to be improved by any other efficient decision tree algorithm.
\end{enumerate}
}

The latter two properties are not provided by statistical consistency \citep{devroye2013probabilistic} or generalization upper bounds \citep{mcallester1999some}. For example, statistical consistency cannot distinguish between variants of ID3 because they are all statistically consistent given that the output trees have unbounded size. \footnote{An algorithm is statistically consistent if the difference between its test error and the bayes-optimal error converges in probability to 0 as the sample size grows to infinity. } Similarly, generalization upper bounds which depend on the tree size cannot distinguish between decision tree algorithms that output trees of the same size.

\section{Conjunctions and Product Distributions}
\label{sec:conjunction}
In this section we provide the first theoretical analysis of MIC. We consider learning a conjunction on $k$ out of $n$ bits with TopDown under a product distribution $\gD$. We will show that $m_{\gD}^{TopDown}(t^*) = 0$ for all $1 \le t^* \le k$. In fact, as the proof will show, this result holds for any variant of ID3. In the next section we will show optimal results for TopDown, which do not hold for other variants of ID3.

\subsection{Setup and Additional Notations} 

Let $J \subset [n]$ be a subset of
indexes such that $|J|=k$. In this section we assume a target function $f_J(\vx) = \bigwedge_{i \in J} x_i$. Note that $f_J$ is realizable by a depth $k$ tree. Let $\mathcal{D}$ be the product distribution on $\mathcal{X} \times \mathcal{Y}$ defined in Section \ref{sec:preliminaries}. We assume, without loss of generality, that $0 < p_1 \le p_2 \le \cdots \le p_n < 1$ and denote $q_i = 1- p_i$. Denote $J = \{i_1,...,i_k\}$ where $p_{i_1} \le p_{i_{2}} \le \cdots \le p_{i_k}$ and define $J_t = \{p_{i_1}, p_{i_2},..., p_{i_t}\}$ for any $1 \le t \le k$ and $J_0 := \emptyset$.

For any tree $T$ let $I_T$ be the set of features that appear in all of its nodes. For simplicity, we define $I_t$ to be the set of features of the tree output by Topdown $T_t$ and let $I_0 := \emptyset$. We say that a binary tree is right-skewed if the left child of each internal node is a leaf with label 0. We denote by $\ma_t$ the set of all right-skewed trees $R \in \mt_t$ such that $I_{R} \subseteq J$.




\ignore{
\begin{proof}
	\begin{enumerate}
		\item Define $f(x) = xH(yx_1) - x_1H(yx)$. Then $f(x_1) = 0$ and by the fact that $H'(x) = -\log\left(\frac{x}{1-x}\right)$ we get for $x_1 \le x \le 1$:
		\begin{align*}
			f'(x) &= -yx_1\log(yx_1) - (1-yx_1)\log(1-yx_1) + yx_1\log\left(\frac{yx}{1-yx}\right) \\ &= -\log(1-yx_1) + yx_1\log\left(\frac{yx(1-yx_1)}{yx_1(1-yx)}\right) > 0
		\end{align*}
		where the last inequality follows since $0 < x_1 \le x$ and $0 < yx_1 < 1$. This completes the proof.
		\item We will consider several cases. If $yx_1 \ge \frac{1}{2}$ then $x_1C(yx_2) \le x_2C(yx_1)$ holds iff $x_1(1-yx_2) \le x_2(1-yx_1)$ which is equivalent to $x_1 \le x_2$. If $yx_2 \le \frac{1}{2}$ then $C(yx_1) = yx_1$ and $C(yx_2) = yx_2$ and the claim holds. Finally, if $yx_1 \le \frac{1}{2} \le yx_2$ then the desired inequality is equivalent to $x_1(1-yx_2) \le x_2yx_1$ which holds since $yx_2 \ge \frac{1}{2}$.
	\end{enumerate}
\end{proof}
}

\subsection{Main Result}

In this section we will provide a partial proof of the following theorem. The remaining details are deferred to the supplementary material.
\begin{thm}
	\label{thm:disj_pop_main}
	Assume that TopDown runs for $1 \le t \le k$ iterations. Then it holds that $\epsilon_t = 0$. Therefore, $m_{\gD}^{TopDown}(t^*) = 0$ for every $1 \le t^* \le k$.
\end{thm}

For the proof we will need the following key lemma which is used throughout our analysis. The proof is given in the supplementary material.
\begin{lem}
	\label{lem:HC_prop}
	Let $0<y<1$ and $0 < x_1 \le x_2 \le 1$. Then: 
	\begin{enumerate}
		\item $x_1H(yx_2) \le x_2H(yx_1)$ and this inequality is strict if $x_1 < x_2$.
		\item $x_1C(yx_2) \le x_2C(yx_1)$.
	\end{enumerate}
\end{lem}

The proof outline of Theorem \ref{thm:disj_pop_main} goes as follows. First, we show that the set of optimal trees in $\mt_t$ intersects with the set $\ma_t$ (Lemma \ref{lem:optimal_tree_conj}). Then we will show that \topdown\,chooses features in $J$ in ascending order of $p_i$ (Lemma \ref{lem:c45_tree}). In Lemma \ref{lem:best_left_skewed_tree} we will prove that the tree found by \topdown\,has the best test error in the set $\ma_t$. By combining all of these facts together we get the theorem.

\begin{lem}
	\label{lem:optimal_tree_conj}
	For any $1 \le t \le k$ there exists a right-skewed tree $R \in \mt_t$ such that $I_{R} \subseteq J$ and $R$ has the lowest test error among all trees in $\mt_t$.
\end{lem}

The proof idea is to use Lemma \ref{lem:HC_prop} to show that any tree in $\mt_t$ can be converted to a right-skewed tree $R \in \mt_t$ such that $I_{R} \subseteq J$, without increasing the test error. The full proof appears in the supplementary material.

\ignore{
\begin{proof}
	Let $S \in \mt_t$ be a tree that has the lowest test error among all trees in $\mt_t$. We will construct from $S$ a right-skewed tree $T \in \mt_t$ such that $I_{T} \subseteq J$ while not increasing the test error. If $I_S \subseteq J$ then we are done. This follows since for each node that is in the right-most path from the root to the right-most leaf in $S$ with feature in $J$, its left sub-tree can be replaced with a left leaf with label $0$, without increasing the test error. This results in a right-skewed tree with at most $t$ internal nodes. By adding more nodes with features in $J$ we cannot increase the test error. To see this, let $l \in J\setminus I_T$ and assume we add $l$ to $T$ as a right child of the right leaf in $T$. Denote by $T'$ the resulting tree. Then, $E(T) = \prod_{j \in I_T}{p_j}C\left(\prod_{j \in J \setminus I_T}{p_j}\right)$ because any right-skewed tree $T$ with $I_T \subseteq J$, can only err in the case that $x_j = 1$ for all $j \in I_T$. Similarly, $E(T') = \prod_{j \in I_T \cup \{l\}}{p_j}C\left(\prod_{j \in J \setminus (I_T \cup \{l\})}{p_j}\right)$. 
	Let $z = \prod_{j \in I_T}{p_j}$, $y = \prod_{j \in J \setminus (I_T \cup \{l\})}p_j$, $x_1 = p_l$ and $x_2 = 1$. Then, by Lemma \ref{lem:HC_prop}, we have $zx_1C(yx_2) \le zx_2C(yx_1)$, which is equivalent to $E(T') \le E(T)$. Therefore, we can get the desired tree $T$.
	
	Now assume that $S$ contains a node with a feature in $[n] \setminus J$. Let $i$ be such a node for which the tree rooted at $i$ contains, besides $i$, only nodes with features in $J$. Denote this sub-tree by $S_i$. Then $S_i$ has the following structure. Without loss of generality, the right sub-tree and the left sub-tree of $i$ are both right-skewed (because otherwise we can replace each with a right leaf with label $0$ without increasing the test error). Consider the following modification to $S_i$. Connect the left sub-tree of $i$ to the right-most leaf of $S_i$, remove the node $i$ and replace it with its right child. Let $v$, be the new right leaf in the tree (that was previously the right leaf of the left sub-tree of $i$). Choose the label for $v$ which results in lowest test error. Finally, remove nodes such that for each feature, there is at most one node with that feature in the path from the root to $v$. Let $T'$ be the tree obtained by this modification to $S$. Then $T'$ has one less node with feature in $[n] \setminus J$ compared to $S$. It remains to show that $E(T') \le E(S)$. This will finish the proof, because we can apply this modification multiple times until we have only features with nodes in $J$. Then we can use the previous argument in the case that $I_S \subseteq J$.
	
	We will now show that $E(T') \le E(S)$. Let $V_1$ be the set of nodes in the path from the root to node $i$ in the tree $S$, excluding $i$. Let $V_2$ be the internal nodes in the right sub-tree of $i$ and $V_3$ be the internal nodes in the left sub-tree of $i$. Recall that the left and right sub-tree are right-skewed. For any node with feature $j$ in $V_1$ let $t_j \in \{p_j,q_j\}$ be the corresponding probability according to the label of $j$ in the path. Then we get the following:
	
	\begin{equation}\
	\label{eq:test_error_diff}
	E(S)-E(T') = p_iD_1 + q_iD_2 - D_3
	\end{equation}
	where 
	\begin{align*}
		D_1 &= \prod_{j \in V_1}{t_j}\prod_{j \in V_2}{p_j}C\left(\prod_{j \in J \setminus ((V_1 \cap J) \cup V_2)}{p_j}\right) \\
		D_2 &= \prod_{j \in V_1}{t_j}\prod_{j \in V_3}{p_j}C\left(\prod_{j \in J \setminus ((V_1 \cap J) \cup V_3)}{p_j}\right) \\
		D_3 &= \prod_{j \in V_1}{t_j}\prod_{j \in V_2 \cup V_3}{p_j}C\left(\prod_{j \in J \setminus ((V_1 \cap J) \cup V_2 \cup V_3)}{p_j}\right)
	\end{align*}
	
	This follows since $p_iD_1$ is the error of the path in $S$ from the root to the right most leaf in the right sub-tree of node $i$.  Similarly, $q_iD_2$ is the error of the path from the root to the right most leaf in the left sub-tree of node $i$ and $D_3$ is the error in the path in $T'$ from the root to the new right leaf. 
	
	Let $z = \prod_{j \in V_1}{t_j}\prod_{j \in V_2}{p_j}$, $y = \prod_{j \in J \setminus ((V_1 \cap J) \cup V_2 \cup V_3)}{p_j}$, $x_1 = \prod_{j \in V_3 \setminus V_2}{p_j}$ and $x_2 = 1$. By Lemma \ref{lem:HC_prop}, it holds that $zx_1C(yx_2) \le zx_2C(yx_1)$, or equivalently, $D_3 \le D_1$. Similarly, we have $D_3 \le D_2$. Hence, by Equation \ref{eq:test_error_diff} we conclude that $E(T') \le E(S)$.
\end{proof}
}

The next lemma shows that TopDown chooses features in $J$ in ascending order of $p_i$ using Lemma \ref{lem:HC_prop}. The proof is given in the supplementary material.
\begin{lem}
	\label{lem:c45_tree}
	Assume that TopDown runs for $1 \le t \le k$ iterations. Then $T_t$ is right-skewed, $I_t = J_t \subseteq J$ and $T_t$ has test error $\prod_{i \in J_t}{p_i}C\left(\prod_{i \in J \setminus J_t}{p_i}\right)$. \footnote{In the case that there are features $i_{l_1}, i_{l_2} \in J$ with $p_{i_{l_1}} = p_{i_{l_2}}$ and $l_1 < l_2$, we assume, without loss of generality, that \topdown\,chooses feature $i_{l_1}$ before $i_{l_2}$.}
\end{lem}
\ignore{
\begin{proof}
	We will first prove by induction that $I_t = J_t$. For the base case $I_0 = J_0 = \emptyset$. Assume that up until iteration $0 \le t < k$, ID3 chose the features $I_t = J_t$. 
	First we note that since feature $i \notin J$ is independent of features in $J$, and $f_J$ depends only on features in $J$, it follows that for any iteration, the gain of feature $i$ is zero. 
	
	Now, for any $l > t$ the gain of feature $i_l \in J$ is 
	\begin{align*}
		&H\left(\prod_{j \in J\setminus J_t}{p_j}\right) - p_{i_l}H\left(\prod_{j \in J \setminus (J_t\cup \{i_l\})}{p_j}\right) + q_{i_l}H(0) \\ &= H\left(\prod_{j \in J\setminus J_t}{p_j}\right) - p_{i_l}H\left(\prod_{j \in J \setminus (J_t\cup \{i_l\})}{p_j}\right) > 0
	\end{align*}
	where the last inequality follows from the concavity of $H$ and the fact that $p_{i_l} > 0$. Therefore, if $t+1 = k$ we are done because $ID3$ will choose feature $i_{k}$ which has the only non-zero gain.
	
	If $t+1 < k$ then let $t+1 < r \le k$. By setting $y = \prod_{j \in J \setminus (J_t\cup \{i_{t+1}, i_r\})}{p_j}$, $x_1 = p_{i_{t+1}}$, $x_2 = p_{i_r}$ and applying Lemma \ref{lem:HC_prop} we have $x_1H(yx_2) \le x_2H(yx_1)$, or equivalently, $p_{i_{t+1}}H\left(\prod_{j \in J \setminus (J_t\cup \{i_{t+1}\})}{p_j}\right) \le p_{i_r}H\left(\prod_{j \in J \setminus (J_t\cup \{i_r\})}{p_j}\right)$ and the inequality is strict if $p_{i_{t+1}} < p_{i_{r}}$. Therefore, $i_{t+1}$ has the largest gain in iteration $t+1$ and ID3 will choose it.
	
	Finally, we note that the latter proof shows that ID3 builds a right-skewed tree. It follows that the test error of $T_t$ is 
	$\prod_{i \in I_t}{p_i}C(\prod_{i \in J \setminus I_t}{p_i})$.
\end{proof}
}

The next lemma shows that the test error of the tree generated by ID3 at iteration $t$, is the lowest among all test errors of trees in $\ma_t$.
\begin{lem}
	\label{lem:best_left_skewed_tree}
	The following equality holds: $$\min_{|I|=t, I \subseteq J}{\prod_{i \in I}{p_i}C\left(\prod_{i \in J \setminus I}{p_i}\right)} = \prod_{i \in J_t}{p_i}C\left(\prod_{i \in J \setminus J_t}{p_i}\right)$$.
\end{lem}
\begin{proof}
	Define $g(K) = \prod_{i \in K}{p_i}C\left(\prod_{i \in J \setminus K}{p_i}\right)$. Let $K_1 \subseteq J$ such that $K_1 \neq J_t$ and $|K_1|=t$. By definition of $K_1$, there exists $j \in K_1$ and $l \in J_t \setminus K_1$ such that $p_j \ge p_l$.  Define $K_2 = (K_1 \setminus \{j\}) \cup \{l\}$. It suffices to prove that $g(K_1) \ge g(K_2)$.  Denote $z = \prod_{i \in K_1 \setminus \{j\}}{p_i}$ and $y = \prod_{i \in J \setminus (K_1 \cup \{l\})}{p_i}$. It holds that $g(K_1) = zp_jC(yp_l)$ and $g(K_2) = zp_lC(yp_j)$. Since $p_j \ge p_l$, $z > 0$ and $0 < y < 1$, we conclude that $g(K_1) \ge g(K_2)$ by Lemma \ref{lem:HC_prop}.
\end{proof}

We are now ready to prove the theorem.

\begin{proof}[Proof of Theorem \ref{thm:disj_pop_main}]
	By Lemma \ref{lem:optimal_tree_conj}, there exists a tree in $R \in \ma_t$ that has optimal test error among all trees in $\mt_t$.
	By Lemma \ref{lem:c45_tree}, TopDown generates a tree $T_t \in \ma_t$ such that $I_t = J_t \subseteq J$. Finally, Lemma \ref{lem:best_left_skewed_tree} implies that $E(T_t) = E(R)$, which proves the theorem.
\end{proof}
\section{Read-Once DNF with 2 Terms and Uniform Distribution}
\label{sec:read_once_theory}

\begin{figure*}[t]
	\begin{subfigure}{.5\textwidth}
		\centering
		\includegraphics[width=1.0\linewidth]{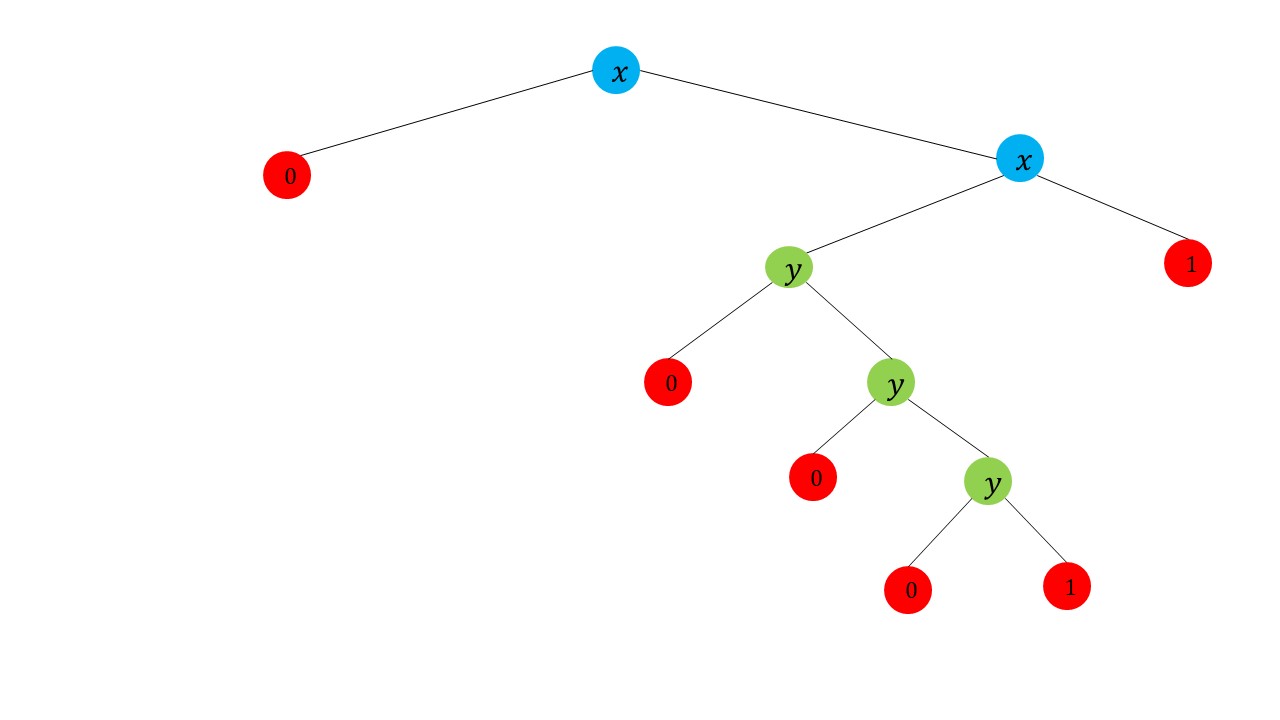}
		\caption{}
		\label{fig:best_first_tree}
	\end{subfigure}%
	\begin{subfigure}{.5\textwidth}
		\centering
		\includegraphics[width=1.0\linewidth]{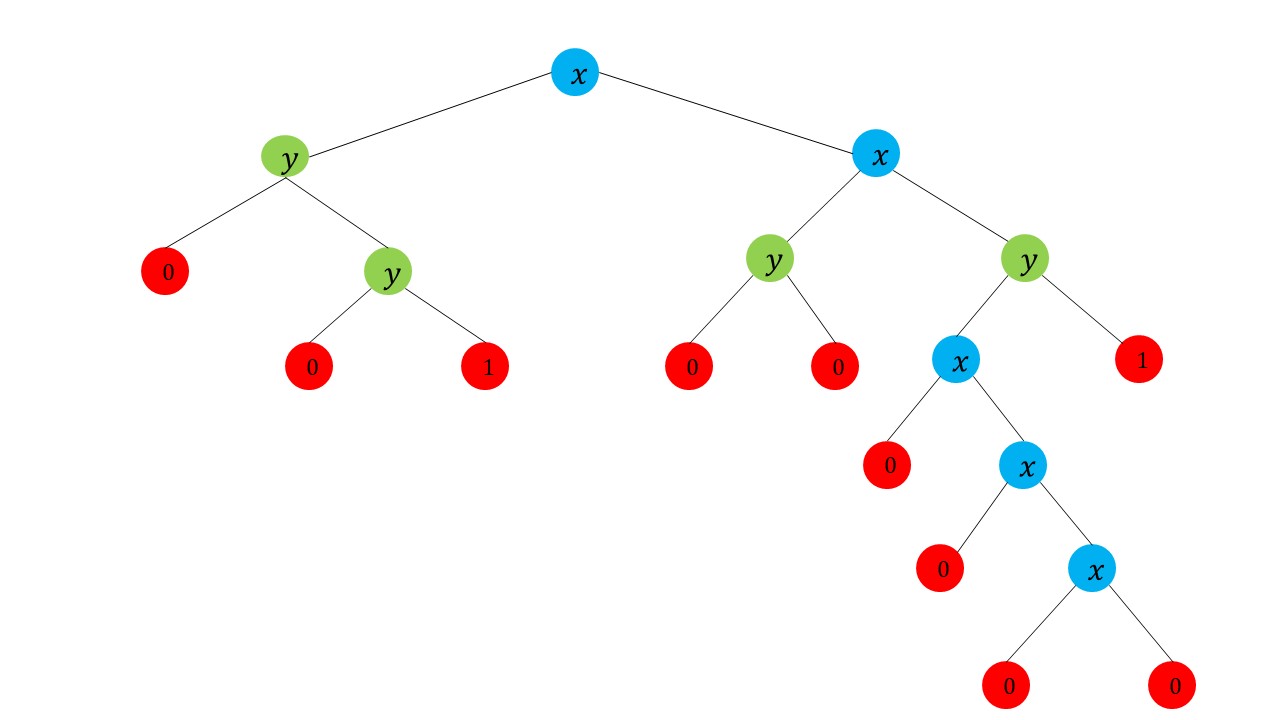}
		\caption{}
		\label{fig:hard_c2}
	\end{subfigure}%
	\caption{\small{Two examples of trees in $\mc_2$. $x$-nodes are in blue, $y$-nodes in green and leaves in red. The left sub-trees of nodes on the right-path are trees in $\mc_1$.}}
	\label{fig:tree_types}
\end{figure*}

In this section, we analyze the \topdown\, algorithm for learning read-once DNFs with 2 terms under the uniform distribution. Similarly to the previous section, we show that for each iteration $t$, TopDown generates a tree with the best test error among all trees with at most $t$ internal nodes. This implies that TopDown attains optimal MIC, i.e., $m_{\gD}^{TopDown}(t^*) = 0$ for all $t^* > 0$. However, as we show in Section \ref{sec:read_once_experiments}, this does not hold for BestFirst \citep{shi2007best} which is a widely used variant of ID3.

\begin{figure*}[th]
	\begin{subfigure}{.5\textwidth}
		\centering
		\includegraphics[width=1.0\linewidth]{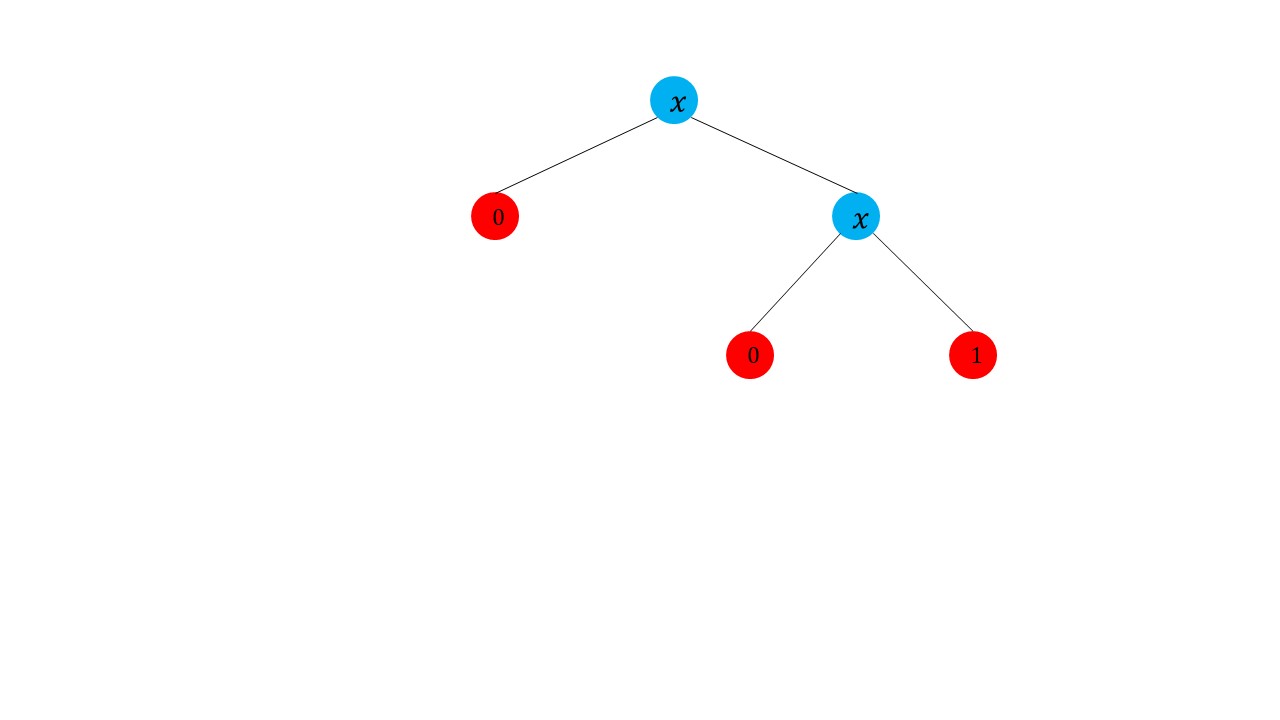}
		\caption{}
		\label{fig:b2}
	\end{subfigure}%
	\begin{subfigure}{.5\textwidth}
		\centering
		\includegraphics[width=1.0\linewidth]{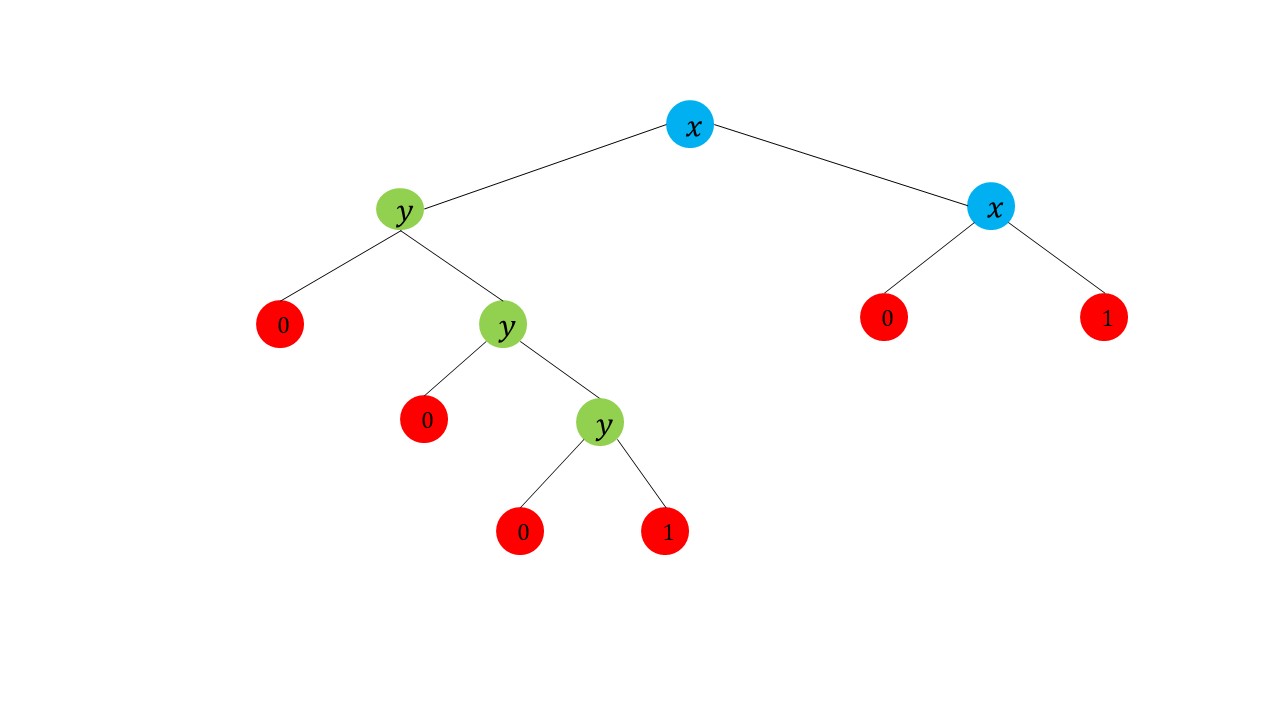}
		\caption{}
		\label{fig:b5}
	\end{subfigure}%
	
	\begin{subfigure}{.5\textwidth}
		\centering
		\includegraphics[width=1.0\linewidth]{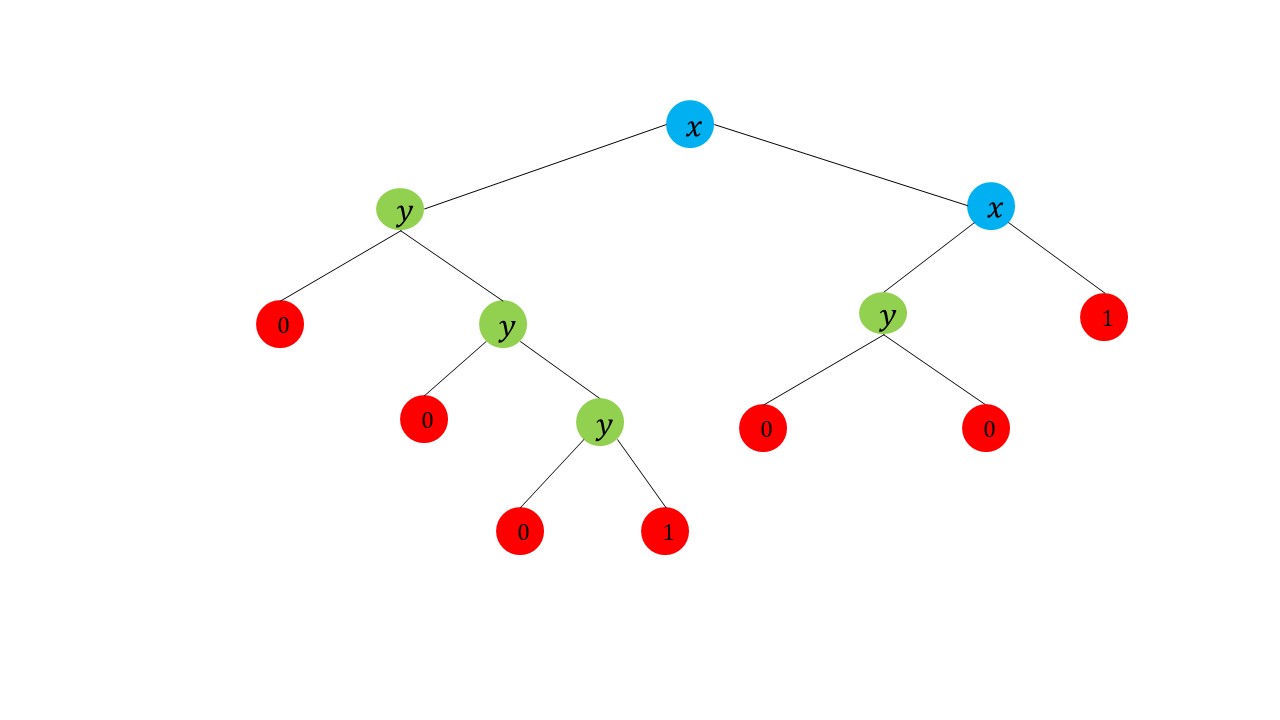}
		\caption{}
		\label{fig:b6}
	\end{subfigure}%
	\begin{subfigure}{.5\textwidth}
		\centering
		\includegraphics[width=1.0\linewidth]{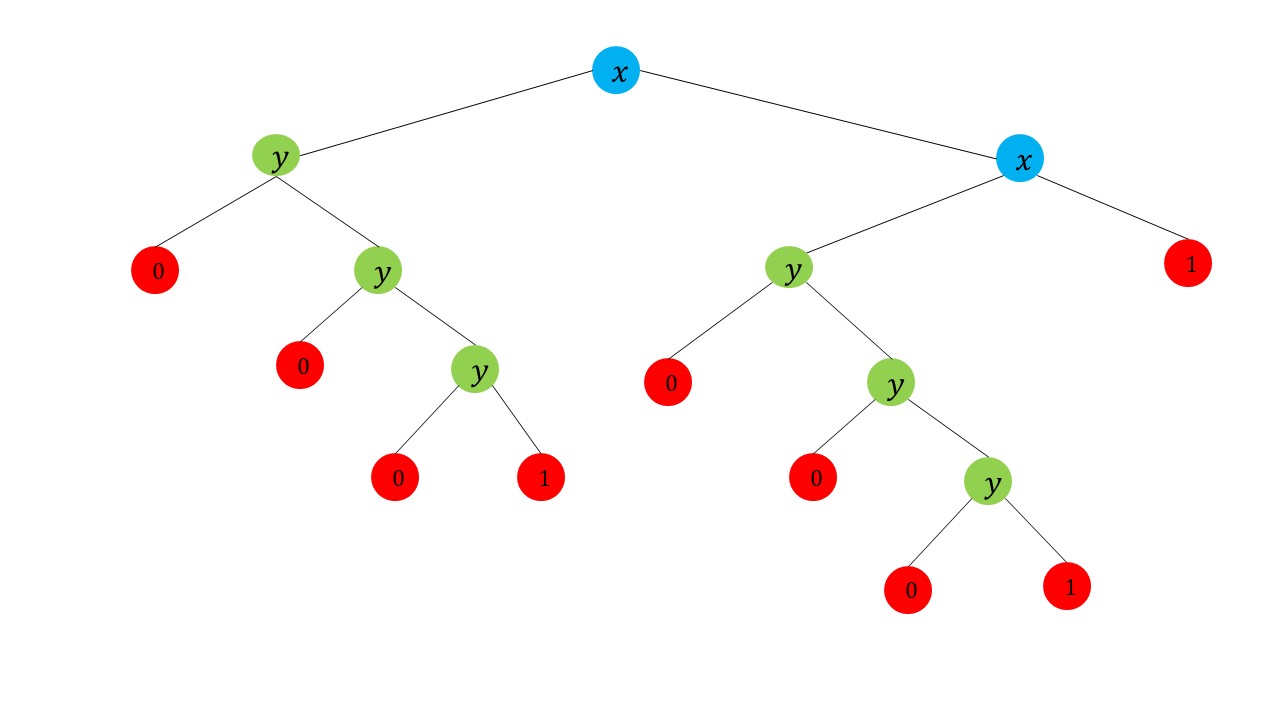}
		\caption{}
		\label{fig:b8}
	\end{subfigure}%
	\caption{\small{$B_t$ examples for DNF $(x_1 \wedge x_2) \vee (y_1 \wedge y_2 \wedge y_3)$. $x$-nodes are in blue, $y$-nodes in green and leaves in red. (a) $t=2$. (b) $t=5$. (c) $t=6$. (d) $t=8$. This tree has 0 test error.}}
	\label{fig:bt}
\end{figure*}

\subsection{Setup and Additional Notations}

\underline{\textbf{Learning Setup:}} We assume that the distribution $\gD$ has a boolean target function $f:\{0,1\}^n \rightarrow \{0,1\}$, where $f(\rvx) = \left(x_{i_1} \wedge \cdots \wedge x_{i_{l}}\right) \vee \left(x_{j_1} \wedge \cdots \wedge x_{j_{m}}\right)$ where $1 \le l \le m$ and all literals are of different variables. WLOG, we only consider literals which are variables and not their negations. By symmetry, our analysis holds for all literal configurations. We assume that the marginal distribution of $\md$ on $\gX$ is the uniform distribiution. For convenience, we will denote $f(\rvx) = (x_1 \wedge x_2 \wedge \cdots \wedge x_l) \vee (y_1 \wedge y_2 \wedge \cdots \wedge y_m)$, where each variable in this formula corresponds to an entry of $\rvx$. We say that the $x_i$ are $x$-variables and similarly define $y$-variables. 

\underline{\textbf{Additional Notations and Definitions}}: Let $T$ be any decision tree. Let $r_T$ be the root of $T$ and $T_L$ and $T_R$ be the left and right sub-trees of $r_T$, respectively. For a node $v$, let $d(v)$ be its depth in the tree, where $d(r_T) = 0$. We let $F(v)$ be the DNF formula corresponding to the node $v$. This is the ground truth DNF $f$ conditioned on all variables and assignments in the path from the root to $v$. If a node $v$ is split with respect to an $x$-variable then we say that $v$ is an $x$-node. Similarly, we define $y$-nodes. We define all nodes which are not $x$-nodes or $y$-nodes as $o$-nodes. From now on, we consider trees whose nodes are one of the latter 3 types.

Let $\mt_t$ be the set of all trees with $t$ nodes. For a node $v$, consider its split with respect to the variable with maximal information gain. Let $v_0$ be its left child and $v_1$ its right child after the split. We define the weighted gain of $v$ as $W(v) = p(v)\left(H(q(v)) -\frac{1}{2}H(q(v_0)) - \frac{1}{2}H(q(v_1))\right)$.
In each iteration, \topdown\,chooses the leaf $l$ for which $W(l)$ is maximal. Equivalently, this is the leaf which maximally decreases $H(T)$. We also define the \textit{error reduction} of the node $v$ to be $C(v) = C(q(v)) -\frac{1}{2}C(q(v_0)) - \frac{1}{2}C(q(v_1))$
and $C(T) = \sum_{v \in \mi(T)}{p(v)C(v)}$ the error reduction of $T$.
In each iteration, in which a node $v$ is split, the test error is decreased by $p(v)C(v)$. Therefore, we get the following identity:
\begin{equation}
\label{eq:et_ct}
E(T) = C(\prob{\md}{f(\rvx)=1}) - C(T)
\end{equation}

\ignore{
\begin{proof}
We will prove the claim by induction on the number of nodes $t$ in the tree. For $t=0$, $\mi(T) = \emptyset$ and the claim holds. Assume it holds for trees with at most $t$ nodes and let $T$ be a tree with $t+1$ nodes. Let $x_j$ be the splitting variable at $r_T$. Then by induction we get
\begin{align*}
E(T) &= \frac{1}{2}E(T_L) + \frac{1}{2}E(T_R) \\ &= \frac{1}{2}C(\prob{\rvx}{f(\rvx)=1\mid x_j=0}) - \frac{1}{2}\sum_{v \in \mi(T_L)}{p_{T_L}(v)C(v)} \\ &+ \frac{1}{2}C(\prob{\rvx}{f(\rvx)=1\mid x_j=1}) - \frac{1}{2}\sum_{v \in \mi(T_R)}{p_{T_R}(v)C(v)} \\ &= C(\prob{\rvx}{f(\rvx)=1}) - C(r_T)  - \sum_{v \in \mi(T) \setminus \{r_T\}}{p(v)C(v)} \\ &= C(\prob{\rvx}{f(\rvx)=1}) - \sum_{v \in \mi(T)}{p(v)C(v)}
\end{align*}
as desired. \textcolor{red}{Change proof.}
\end{proof}
}

By \Eqref{eq:et_ct}, we can reason about $E(T)$ through $C(T)$. For $E(T)$ to be minimal we need $C(T)$ to be maximal.

For a tree $T$ we define its \textit{right-path} to be the nodes in its right-most path. If the right-path consists only of $x$-nodes, we say that it is a right $x$-path. Similarly, we define a right $y$-path. We define $\mc_1$ to be the set of all trees that consist only of a right-path where the nodes are either all $x$ nodes or all $y$-nodes. We also say that these trees are right-paths. We define $\mc_2$ to be the set of all trees such that for each node in the right-path the following holds. If it is an $x$-node, then its left sub-tree is a tree in $\mc_1$ with  $y$-nodes. Similarly, if it is a $y$-node, then its right sub-tree is a tree in $\mc_1$ with $x$-nodes. In Figure \ref{fig:tree_types} we illustrate these sets of trees. For a tree $T\in \mc_1$ such that $F(r_T) = y_1 \wedge y_2 \wedge \cdots \wedge y_m$ we say that $T$ is a full right $y$-path if $T$ has $m$ $y$-nodes.

\subsubsection{Main Result}

We will show that for each number of iterations $t$, \topdown\,builds a tree, which we denote by $B_t$, and this tree has the best test error among all trees with at most $t$ internal nodes. Formally, $B_t$ is a tree of size $t$ in $\mc_2$ whose right-path consists only of $x$-nodes. The left sub-tree of each $x$-node is in $\mc_1$ and consists only of $y$-nodes. Furthermore, for any two $x$-nodes $v_1$ and $v_2$ in the right-path such that $v_2$ is deeper than $v_1$, the following holds. The left sub-tree of $v_2$ is not a leaf only if the left sub-tree of $v_2$ is a full right $y$-path. Figure \ref{fig:bt} shows $B_t$ for the formula $(x_1 \wedge x_2) \vee (y_1 \wedge y_2 \wedge y_3)$. We state our main result in the following theorem:

\begin{thm}
	\label{thm:read_once_main}
	Let $t \ge 1$. Then, $T_t=B_t$ and $T_t$ is a tree with the optimal test error in $\mt_t$.
\end{thm}

From this result we immediately get the following:

\begin{cor}
	For any $t^* > 0$, $m_{\gD}^{TopDown}(t^*) = 0$.
\end{cor}

\begin{figure*}[t]
	\begin{subfigure}{.33\textwidth}
		\centering
		\includegraphics[width=1.0\linewidth]{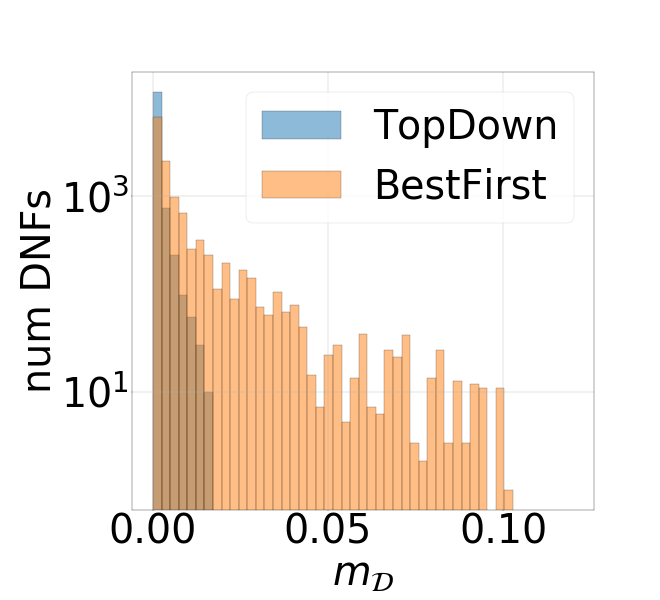}
		\caption{}
		\label{fig:hist_uni}
	\end{subfigure}%
	\begin{subfigure}{.33\textwidth}
		\centering
		\includegraphics[width=1.0\linewidth]{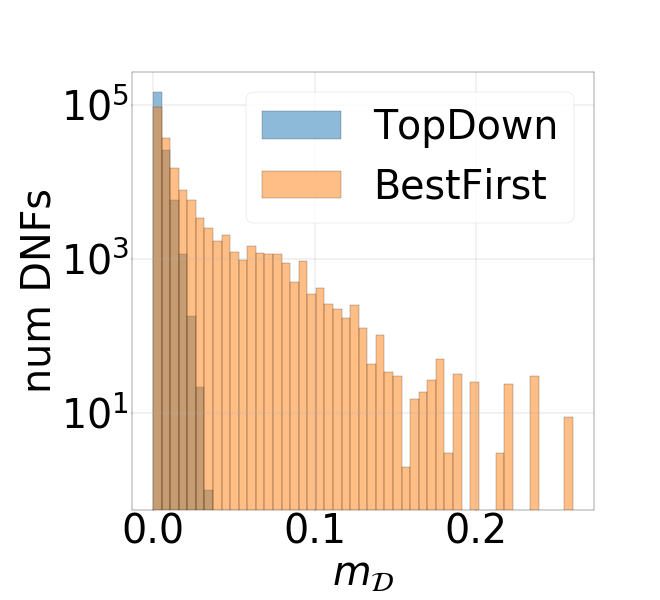}
		\caption{}
		\label{fig:hist_0307}
	\end{subfigure}
	\begin{subfigure}{.33\textwidth}
		\centering
		\includegraphics[width=1.0\linewidth]{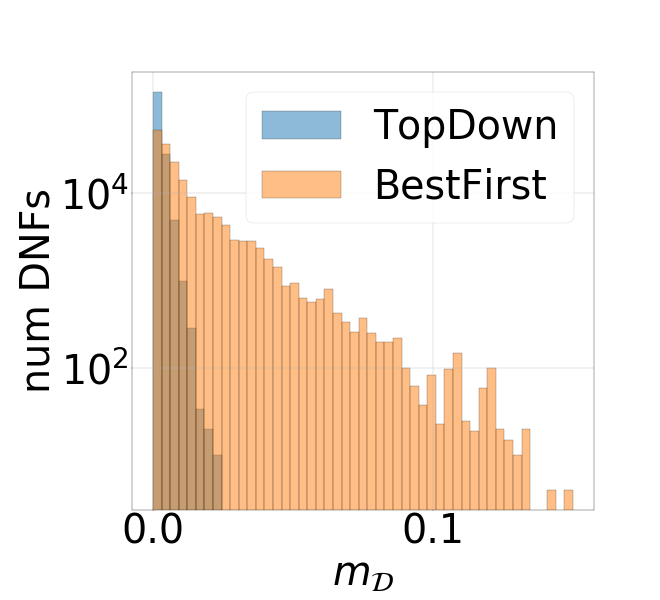}
		\caption{}
		\label{fig:hist_0406}
	\end{subfigure}
	\caption{\small{Empirical results for TopDown and BestFirst. The plots show histogram of $m_{\gD}^{TopDown}$ values for all DNFs in the corresponding setting. The $y$-axis is in log-scale. The plot shows performance of TopDown and BestFirst under the (a) uniform distribution. (b) product distribution with $p_1 = 0.3$, $p_2 = 0.7$. (c) product distribution with $p_1 = 0.4$, $p_2 = 0.6$.}}
	\label{fig:hists}
\end{figure*}

\subsubsection{Proof Sketch of Theorem \ref{thm:read_once_main}}
\label{sec:proof_main_read_once}

The proof proceeds as follows. In the first part we show that for each iteration $t$, \topdown\,outputs the tree $B_t$, i.e., $T_t=B_t$ (Proposition \ref{prop:topdown_bt}). In the second part, we show that for any $t$, $B_t$ has minimum test error among all trees in $\mt_t$ (Proposition \ref{prop:bt_optimal}). These two parts together prove  Theorem \ref{thm:read_once_main}.

We begin with the first part:
\begin{prop}
	\label{prop:topdown_bt}
	Assume \topdown\,runs for $t$ iterations. Then $T_t = B_t$.
\end{prop}

The proof uses a result of \citep{fiat2004decision}, which show that in the setting of this section, for each node $v$ that ID3 splits, it chooses a variable in $F(v)$ which is in a minimal size term. For example, if $F(v) = (x_1 \wedge x_2) \vee (x_3 \wedge x_4 \wedge x_5)$, then ID3 chooses either $x_1$ or $x_2$ (they have the same gain due to the uniform distribution assumption). Then, the proof follows by several inequalities involving the entropy function. These inequalities arise from comparing the weighted gain of pairs of nodes and showing that their correctness implies that $T_t = B_t$. We defer the proof to the supplementary material.

Next, we show the following proposition.
	
\begin{prop}
\label{prop:bt_optimal}
	Let $t \ge 1$. Then, $B_t$ is a tree with optimal test error among all trees in $\mt_t$.
\end{prop}

The idea of the proof is to first show by induction on $t$ that there exists an optimal tree in $\mc_2$. Then, the proof proceeds by showing that any tree in $\mc_2$ with $t$ internal nodes can be converted to $B_t$ without increasing the test error. To illustrate the latter part with a simple example, consider the case where $f(\rvx) =  (x_1 \wedge x_2) \vee (y_1 \wedge y_2 \wedge y_3)$. In this case the tree in Figure \ref{fig:b5}, which we denote by $T_1$, is equal to $B_5$, whereas the tree in Figure \ref{fig:best_first_tree}, which we denote by $T_2$, is a tree in $\mc_2$ with 5 internal nodes which is not $B_5$. In this example, by direct calculation it can be shown that $E(T_1) < E(T_2)$. However, to illustrate our proof in the general case, let $v_1$ be the left child of the root in $T_1$ and let $v_2$ be the left child of the right child of the root in $T_2$. Then $C(v_1) = C(v_2)$, but $p(v_1) > p(v_2)$. Therefore, $p(v_1)C(v_1) > p(v_2)C(v_2)$. By continuing this way for the rest of the nodes on the left sub-trees, we get $C(T_1) > C(T_2)$. This implies by \eqref{eq:et_ct} that $E(T_1) < E(T_2)$. This technique of comparing error reduction of nodes allows us to handle more complex cases, e.g., to show that the tree in Figure \ref{fig:hard_c2} is not optimal. The full proof is given in the supplementary material.

\section{Efficient Calculation of Optimal Trees}
\label{sec:dp}

In this section, we present dynamic programming algorithms that calculate optimal trees efficiently in a large number of settings. We provide algorithms for uniform distributions and certain product distributions. We use these algorithms in Section \ref{sec:read_once_experiments}, to efficiently calculate MSI.

\subsection{Uniform Distribution}

In this section we assume that $\md$ is the uniform distribution. Let $F$ be a read-once dnf with $k$ terms $c_1,...,c_k$. We refer to $F$ as a set over the terms. Since we are dealing with the uniform distribution, the identity of the variables in the terms do not need to be taken into account for the calculation, only the terms' sizes. For a term $c$, let $c^-$ be a term with $\left|c^-\right| = \left|c\right| - 1$. For any $F$, let $OPT(F,t)$ be the minimal test error of all trees with at most $t$ internal nodes. We use the following relation to compute the test error of optimal trees:
\begin{align*}
OPT(F,t) &= \min_{1\le i \le k,\,0 \le j \le t-1} \Big\{\frac{1}{2}OPT(F\setminus\{c_i\},j) \\ &+\frac{1}{2}OPT(\left(F\setminus\{c_i\}\right)\cup c_i^-,t-1-j)\Big\}
\end{align*}

The correctness of the formula follows since $\md$ is a product distribution and a sub-tree of an optimal tree is optimal. See supplemantary for details. \footnote{Note that we do not need to consider trees with variables that are not in the DNF. Details are provided in the supplementary.} The computational complexity of this procedure is $O(\prod_{i}|c_i|t)$, whereas the total number of decision trees with $t$ inner nodes is at least $t!2^t$.
\subsection{Product Distributions}
\label{sec:product_exps}

In this section we assume that the distribution $\md$ over variables is a product distribution where each variable has distribution $Bernoulli(p_1)$ or $Bernoulli(p_2)$ for $0 < p_1, p_2 < 1$. Let $F$ be a read-once dnf with $k$ terms $c_1,...,c_k$. We refer to $F$ as a set over the terms. For each term $c$, let $n_{c,1}$ be the number of variables with disribution $Bernoulli(p_1)$ and similarly define $n_{c,2}$. Denote by $c(n_1,n_2)$ a term with $n_1$ variables with distribution $Bernoulli(p_1)$ and $n_2$ variables with distribution $Bernoulli(p_2)$. 

Define $OPT(F,t)$ as in the previous section. We use the following relation to compute the test error of optimal trees:

\begin{align*}
\label{eq:product_exp_formula}
& OPT(F,t)  = \\ &\min_{1\le i \le k,\,0 \le j \le t-1}\big\{\min\big\{S_1(F,t,j,c_i),S_2(F,t,j,c_i) \big\}\big\} \numberthis
\end{align*}
where
\begin{align*}
&S_1(F,t,j,c^*) = \left(1-p_1\right)OPT(F\setminus\{c^*\},j) \\& + p_1 OPT(\left(F\setminus\{c^*\}\right)\cup c(n_{c^*,1}-1, n_{c^*,2}),t-1-j)
\end{align*}
and
\begin{align*}
&S_2(F,t,j,c^*) = \left(1-p_2\right)OPT(F\setminus\{c^*\},j) \\ &+ p_2 OPT(\left(F\setminus\{c^*\}\right)\cup c(n_{c^*,1}, n_{c^*,2}-1),t-1-j)
\end{align*}

The correctness of the formula follows similarly as in the uniform case. See supplementary for details.

\section{Empirical Results}
\label{sec:read_once_experiments}

In this section we present our empirical results. We first introduce the BestFirst algorithm and show how it can build a non-optimal tree based on our theoretical results. Then, we apply the algorithms in Section \ref{sec:dp} to compute the MIC in a large number of settings for TopDown and BestFirst. Finally, we show we that our MIC analysis predicts well the test performance gap between TopDown and BestFirst.

\begin{figure*}[t]
	\begin{subfigure}{.33\textwidth}
		\centering
		\includegraphics[width=1.0\linewidth]{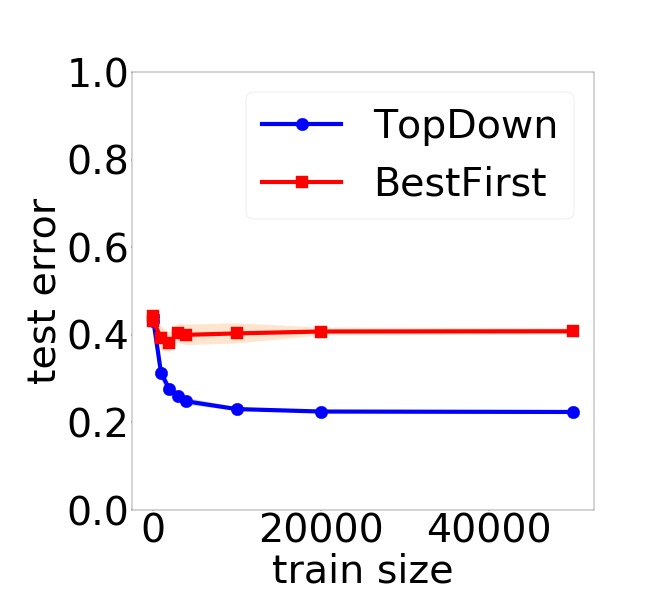}
		\caption{}
		\label{fig:finite_uni}
	\end{subfigure}%
	\begin{subfigure}{.33\textwidth}
		\centering
		\includegraphics[width=1.0\linewidth]{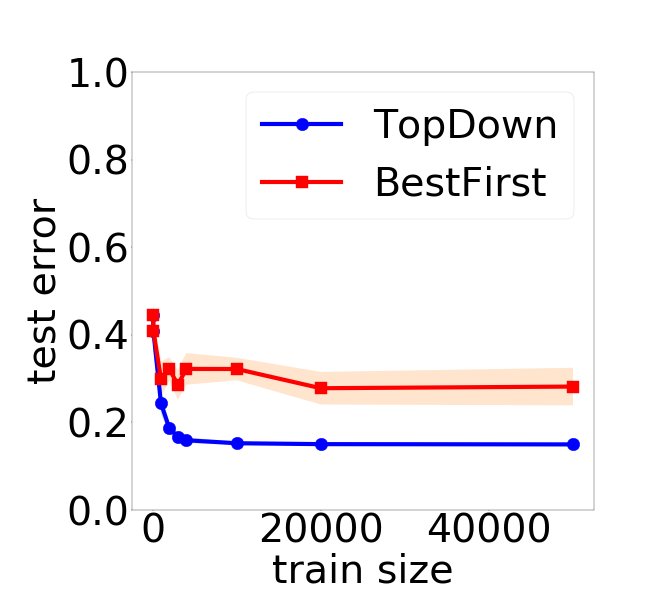}
		\caption{}
		\label{fig:finite_0307}
	\end{subfigure}
	\begin{subfigure}{.33\textwidth}
		\centering
		\includegraphics[width=1.0\linewidth]{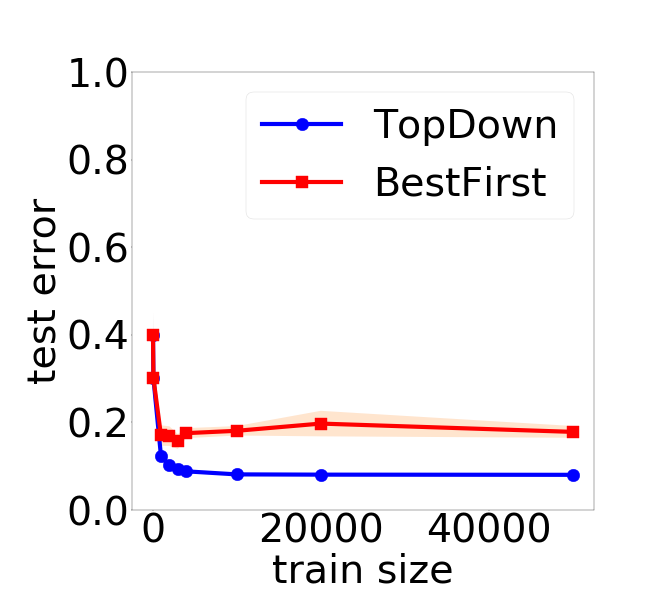}
		\caption{}
		\label{fig:finite_0406}
	\end{subfigure}
	\caption{\small{Empirical results for TopDown and BestFirst on finite samples. The plots show test error as a function of the training set size for marginal distribution: (a) uniform distribution. (b) product distribution with $p_1 = 0.3$, $p_2 = 0.7$. (c) product distribution with $p_1 = 0.4$, $p_2 = 0.6$. }}
	\label{fig:finite}
\end{figure*}

\subsection{BestFirst}
\label{sec:best_first}

The BestFirst algorithm \citep{shi2007best}, is similar to TopDown but with a different policy to choose leaves in each iteration.  A version of BestFirst is used in WEKA \citep{hall2009weka}. Instead of choosing the leaf and feature with maximal weighted gain (Eq. \ref{eq:weighted_gain}), it chooses the leaf $l$ and feature $i$ with maximal gain, i.e, which maximize $H(q(l)) -(1-\tau_i) H(q(l_0)) - \tau_i H(q(l_1))$. This can degrade the performance compared to TopDown. For example, consider learning the formula $(x_1 \wedge x_2) \vee (y_1 \wedge y_2 \wedge y_3)$ under the uniform distribution. As shown in Section \ref{sec:read_once_theory}, after 5 iterations, TopDown will generate the tree in Figure \ref{fig:b5} and this is the optimal tree with 5 internal nodes. However, it can be shown that BestFirst can generate the tree in Figure \ref{fig:best_first_tree} after 5 iterations (see supplementary material for details). As shown in Section \ref{sec:proof_main_read_once}, the latter tree is sub-optimal.

\subsection{MIC Calculation}

We used the algorithms in Section \ref{sec:dp} to compute the MIC for the uniform distribution and product distributions. For the uniform distribution, let $\gF$ be the set of all read-once DNFs over the uniform distribution with at most 8 terms and most 8 literals in each term. We calculated $OPT(F,t)$ for all $1 \le t \le 100$ and for all $F \in \gF$. For each $F \in \gF$ we calculated $m_{\gD}^{TopDown}(100)$. For the product distributions, we experimented with the pairs $(p_1,p_2) \in \{(0.3,0.7),(0.4,0.6)\}$. Given a pair $(p_1,p_2)$, let $\mb$ be the set of all read-once DNFs over $\md$ with at most 5 terms and at most 5 literals, where each literal is a variable (not its negation) which has distribution $Bernoulli(p_1)$ or $Bernoulli(p_2)$. We calculated $OPT(F,t)$ for all $1 \le t \le 100$ and for all $F \in \mb$. For each $F \in \mb$ we calculated $m_{\gD}^{TopDown}(100)$.

The results are given in Figure \ref{fig:hists}. They show that for the vast majority of read-once DNFs, \topdown\ is near-optimal with $m_{\gD}^{TopDown}(100) << 0.01$.  On the other hand, for many DNFs the trees generated by BestFirst can be far from optimal with $m_{\gD}^{TopDown}(100) > 0.05$.



\subsection{Finite Sample Experiments}

Here we compare the performance of TopDown and BestFirst in a typical finite sample setting. We experimented with a DNF with 10 terms and $n=38$. We ran both TopDown and BestFirst for at most 20 iterations and split a leaf only if there were at least 20 points that reach the leaf. Full details of the experiments are provided in the supplementary. Figure \ref{fig:finite} shows the results. It can be seen that TopDown significantly outperforms BestFirst, as predicted by our MIC analysis. 
\section{Conclusion}

In this work, we introduce MIC, which is a  novel metric for analyzing decision tree algorithms. We provide theoretical and empirical evidence that TopDown has near-optimal MIC for a large number settings for learning read-once DNFs under product distributions. MIC presents a novel view of decision tree algorithms and allows to compare their performance in practical settings. 

There are many interesting directions for future work. First, it would be interesting to close the gap between our theory and experiments. We conjecture that in most cases, TopDown generates near-optimal trees for learning read-once DNFs under product distributions. It would be interesting to consider other distributions with dependencies between variables and other DNFs. Furthermore, MIC might be helpful in analyzing other learning algorithms. For example, providing theoretical guarantees for random forests of practical size using the ideas presented here, is an interesting direction for future work. 

\bibliography{topdown}
\bibliographystyle{icml2019}

\clearpage
\appendix
\onecolumn
\icmltitle{Supplementary material}

\section{Proof of Theorem \ref{thm:nph}}

In \citep{hancock1996lower}, it is shown that the problem of finding a consistent decision tree of minimal size is NP-Hard. Denote this problem $P$. We will show a poly-time reduction from $P$ to our problem.  Given an instance of $P$ with training set $S$, let $\gD$ be the uniform distribution over elements in $S$. Then, for  $1 \le t \le n\left|S\right|$ find the tree $T_t$ with test error $OPT(t,\gD)$. Among these trees, return the tree with minimal $t$ such that $OPT(t,\gD)=0$. Since the minimal-size decision tree which is consistent with $S$ has at most $n\left|S\right|$ internal nodes, the claim follows. 

\section{Proofs for Section \ref{sec:conjunction}}
\subsection{Proof of Lemma \ref{lem:HC_prop}}
\begin{enumerate}
	\item Define $f(x) = xH(yx_1) - x_1H(yx)$. Then $f(x_1) = 0$ and by the fact that $H'(x) = -\log\left(\frac{x}{1-x}\right)$ we get for $x_1 \le x \le 1$:
	\begin{align*}
	f'(x) &= -yx_1\log(yx_1) - (1-yx_1)\log(1-yx_1) + yx_1\log\left(\frac{yx}{1-yx}\right) \\ &= -\log(1-yx_1) + yx_1\log\left(\frac{yx(1-yx_1)}{yx_1(1-yx)}\right) > 0
	\end{align*}
	where the last inequality follows since $0 < x_1 \le x$ and $0 < yx_1 < 1$. This completes the proof.
	\item We will consider several cases. If $yx_1 \ge \frac{1}{2}$ then $x_1C(yx_2) \le x_2C(yx_1)$ holds iff $x_1(1-yx_2) \le x_2(1-yx_1)$ which is equivalent to $x_1 \le x_2$. If $yx_2 \le \frac{1}{2}$ then $C(yx_1) = yx_1$ and $C(yx_2) = yx_2$ and the claim holds. Finally, if $yx_1 \le \frac{1}{2} \le yx_2$ then the desired inequality is equivalent to $x_1(1-yx_2) \le x_2yx_1$ which holds since $yx_2 \ge \frac{1}{2}$.
\end{enumerate}

\subsection{Proof of Lemma \ref{lem:optimal_tree_conj}}

Let $S \in \mt_t$ be a tree that has the lowest test error among all trees in $\mt_t$. We will construct from $S$ a right-skewed tree $T \in \mt_t$ such that $I_{T} \subseteq J$ while not increasing the test error. If $I_S \subseteq J$ then we are done. This follows since for each node that is in the right-most path from the root to the right-most leaf in $S$ with feature in $J$, its left sub-tree can be replaced with a left leaf with label $0$, without increasing the test error. This results in a right-skewed tree with at most $t$ internal nodes. By adding more nodes with features in $J$ we cannot increase the test error. To see this, let $l \in J\setminus I_T$ and assume we add $l$ to $T$ as a right child of the right leaf in $T$. Denote by $T'$ the resulting tree. Then, $E(T) = \prod_{j \in I_T}{p_j}C\left(\prod_{j \in J \setminus I_T}{p_j}\right)$ because any right-skewed tree $T$ with $I_T \subseteq J$, can only err in the case that $x_j = 1$ for all $j \in I_T$. Similarly, $E(T') = \prod_{j \in I_T \cup \{l\}}{p_j}C\left(\prod_{j \in J \setminus (I_T \cup \{l\})}{p_j}\right)$. 
Let $z = \prod_{j \in I_T}{p_j}$, $y = \prod_{j \in J \setminus (I_T \cup \{l\})}p_j$, $x_1 = p_l$ and $x_2 = 1$. Then, by Lemma \ref{lem:HC_prop}, we have $zx_1C(yx_2) \le zx_2C(yx_1)$, which is equivalent to $E(T') \le E(T)$. Therefore, we can get the desired tree $T$.

Now assume that $S$ contains a node with a feature in $[n] \setminus J$. Let $i$ be such a node for which the tree rooted at $i$ contains, besides $i$, only nodes with features in $J$. Denote this sub-tree by $S_i$. Then $S_i$ has the following structure. Without loss of generality, the right sub-tree and the left sub-tree of $i$ are both right-skewed (because otherwise we can replace each with a right leaf with label $0$ without increasing the test error). Consider the following modification to $S_i$. Connect the left sub-tree of $i$ to the right-most leaf of $S_i$, remove the node $i$ and replace it with its right child. Let $v$, be the new right leaf in the tree (that was previously the right leaf of the left sub-tree of $i$). Choose the label for $v$ which results in lowest test error. Finally, remove nodes such that for each feature, there is at most one node with that feature in the path from the root to $v$. Let $T'$ be the tree obtained by this modification to $S$. Then $T'$ has one less node with feature in $[n] \setminus J$ compared to $S$. It remains to show that $E(T') \le E(S)$. This will finish the proof, because we can apply this modification multiple times until we have only features with nodes in $J$. Then we can use the previous argument in the case that $I_S \subseteq J$.

We will now show that $E(T') \le E(S)$. Let $V_1$ be the set of nodes in the path from the root to node $i$ in the tree $S$, excluding $i$. Let $V_2$ be the internal nodes in the right sub-tree of $i$ and $V_3$ be the internal nodes in the left sub-tree of $i$. Recall that the left and right sub-tree are right-skewed. For any node with feature $j$ in $V_1$ let $t_j \in \{p_j,q_j\}$ be the corresponding probability according to the label of $j$ in the path. Then we get the following:

\begin{equation}\
\label{eq:test_error_diff}
E(S)-E(T') = p_iD_1 + q_iD_2 - D_3
\end{equation}
where 
\begin{align*}
D_1 &= \prod_{j \in V_1}{t_j}\prod_{j \in V_2}{p_j}C\left(\prod_{j \in J \setminus ((V_1 \cap J) \cup V_2)}{p_j}\right) \\
D_2 &= \prod_{j \in V_1}{t_j}\prod_{j \in V_3}{p_j}C\left(\prod_{j \in J \setminus ((V_1 \cap J) \cup V_3)}{p_j}\right) \\
D_3 &= \prod_{j \in V_1}{t_j}\prod_{j \in V_2 \cup V_3}{p_j}C\left(\prod_{j \in J \setminus ((V_1 \cap J) \cup V_2 \cup V_3)}{p_j}\right)
\end{align*}

This follows since $p_iD_1$ is the error of the path in $S$ from the root to the right most leaf in the right sub-tree of node $i$.  Similarly, $q_iD_2$ is the error of the path from the root to the right most leaf in the left sub-tree of node $i$ and $D_3$ is the error in the path in $T'$ from the root to the new right leaf. 

Let $z = \prod_{j \in V_1}{t_j}\prod_{j \in V_2}{p_j}$, $y = \prod_{j \in J \setminus ((V_1 \cap J) \cup V_2 \cup V_3)}{p_j}$, $x_1 = \prod_{j \in V_3 \setminus V_2}{p_j}$ and $x_2 = 1$. By Lemma \ref{lem:HC_prop}, it holds that $zx_1C(yx_2) \le zx_2C(yx_1)$, or equivalently, $D_3 \le D_1$. Similarly, we have $D_3 \le D_2$. Hence, by Equation \ref{eq:test_error_diff} we conclude that $E(T') \le E(S)$.

\subsection{Proof of Lemma \ref{lem:c45_tree}}
We will first prove by induction that $I_t = J_t$. For the base case $I_0 = J_0 = \emptyset$. Assume that up until iteration $0 \le t < k$, ID3 chose the features $I_t = J_t$. 
First we note that since feature $i \notin J$ is independent of features in $J$, and $f_J$ depends only on features in $J$, it follows that for any iteration, the gain of feature $i$ is zero. 

Now, for any $l > t$ the gain of feature $i_l \in J$ is 
\begin{align*}
&H\left(\prod_{j \in J\setminus J_t}{p_j}\right) - p_{i_l}H\left(\prod_{j \in J \setminus (J_t\cup \{i_l\})}{p_j}\right) + q_{i_l}H(0) \\ &= H\left(\prod_{j \in J\setminus J_t}{p_j}\right) - p_{i_l}H\left(\prod_{j \in J \setminus (J_t\cup \{i_l\})}{p_j}\right) > 0
\end{align*}
where the last inequality follows from the concavity of $H$ and the fact that $p_{i_l} > 0$. Therefore, if $t+1 = k$ we are done because TopDown will choose feature $i_{k}$ which has the only non-zero gain.

If $t+1 < k$ then let $t+1 < r \le k$. By setting $y = \prod_{j \in J \setminus (J_t\cup \{i_{t+1}, i_r\})}{p_j}$, $x_1 = p_{i_{t+1}}$, $x_2 = p_{i_r}$ and applying Lemma \ref{lem:HC_prop} we have $x_1H(yx_2) \le x_2H(yx_1)$, or equivalently, $p_{i_{t+1}}H\left(\prod_{j \in J \setminus (J_t\cup \{i_{t+1}\})}{p_j}\right) \le p_{i_r}H\left(\prod_{j \in J \setminus (J_t\cup \{i_r\})}{p_j}\right)$ and the inequality is strict if $p_{i_{t+1}} < p_{i_{r}}$. Therefore, $i_{t+1}$ has the largest gain in iteration $t+1$ and TopDown will choose it.

Finally, we note that the latter proof shows that TopDown builds a right-skewed tree. It follows that the test error of $T_t$ is 
$\prod_{i \in I_t}{p_i}C(\prod_{i \in J \setminus I_t}{p_i})$.

\section{Proofs for Section \ref{sec:read_once_theory}}
\subsection{Proof of Proposition \ref{prop:topdown_bt}}
We first prove several inequalities which involve the entropy function.

\begin{lem}
	\label{lem:a_identity}
	Let $0 < a < \frac{1}{5}$, then
	$$2H(3a) - 3H(a) + H(2a) - H(5a) > 0$$
\end{lem}
\begin{proof}
	Define $g(a) = 2H(3a) - 3H(a) + H(2a) - H(5a)$. Since $g(0) = 0$, it suffices to prove that $g'(a) > 0$ for $0 < a < \frac{1}{5}$. We have, 
	\begin{align*}
	g'(a) &= -6\log\left(\frac{3a}{1-3a}\right) + 3\log\left(\frac{a}{1-a}\right) - 2\log\left(\frac{2a}{1-2a}\right) +5\log\left(\frac{5a}{1-5a}\right) \\ &= \log\left(\frac{(1-3a)^6 a^3 (1-2a)^2 (5a)^5}{(3a)^6(1-a)^3 (2a)^2 (1-5a)^5}\right)
	\end{align*}
	
	First, we notice that $a^3(5a)^5 > (3a)^6(2a)^2$. Therefore, we are left to show that
	\begin{equation}
	\label{eq:a_identity}
	\frac{(1-3a)^6(1-2a)^2}{(1-a)^3(1-5a)^5} > 1
	\end{equation}
	By the following 2 inequalities:
	\begin{enumerate}
		\item $(1-3a)(1-2a) > (1-a)(1-5a)$.
		\item $(1-3a)^2 > (1-a)(1-5a)$.
	\end{enumerate}
	proving \Eqref{eq:a_identity} reduces to showing that $\frac{(1-3a)^2}{(1-5a)^2} > 1$, which is true, as desired.
	
\end{proof}

\begin{lem}
	\label{lem:a_identity2}
	Let $0 < a = \frac{1}{2^k}$, $k \ge 2$ and $H(x) = -x\log(x) - (1-x)\log(1-x)$. Then,
	$$2H(3a-2a^2) - 3H(a) + H(2a) - H(5a-4a^2) > 0$$
\end{lem}
\begin{proof}
	Define $g(a) = 2H(3a-2a^2) - 3H(a) + H(2a) - H(5a-4a^2)$. We will prove that the inequality holds for all $0 < a \le \frac{1}{512}$. The inequality can be proved to hold for the cases $2 \le k \le 8$ by calcaluting $g(a)$ with sufficiently high precision.
	
	Since $g(0) = 0$. It suffices to show that $g'(a) > 0$ for $0 < a \le \frac{1}{512}$. We have, 
	\begin{align*}
	g'(a) &= -2(3-4a)\log\left(\frac{3a-2a^2}{1-(3a-2 a^2)}\right) +3\log\left(\frac{a}{1-a}\right) - \\ &- 2\log\left(\frac{2a}{1-2a}\right) + (5-8a)\log\left(\frac{5a-4a^2}{1-(5a-4a^2)}\right) \\ &= \log\left(\frac{\left((1-a)(1-2a)\right)^{6-8a}a^3(1-2a)^2\left(5a-4a^2\right)^{5-8a}}{\left(3a-2a^2\right)^{6-8a}(1-a)^3 (2a)^2\left((1-a)(1-4a)\right)^{5-8a}}\right) \\ &= \log\left(\frac{(1-2a)^{8-8a}\left(5-4a\right)^{5-8a}}{4\left(3-2a\right)^{6-8a}(1-a)^2 \left(1-4a\right)^{5-8a}}\right)
	\end{align*}
	Define $$\Delta = (1-2a)^{8-8a}\left(5-4a\right)^{5-8a} - 4\left(3-2a\right)^{6-8a}(1-a)^2 \left(1-4a\right)^{5-8a}$$
	It suffices to show that $\Delta > 0$. Define $h_1(x) = \left(3-2x\right)^{6-8x}$ and $h_2(x) = -5120x + 3^6$. It holds that $h_1(0) = h_2(0) = 3^6$ and $h_2\left(\frac{1}{512}\right) = 719 > 3^{6-\frac{1}{64}} > h_1\left(\frac{1}{512}\right)$. Furthermore, $h_1(x)$ is convex for $0 \le x \le \frac{1}{512}$, because in this case:
	\begin{align*}
	\frac{\partial^2 h_1}{\partial x^2}\left(x\right) = \left(3-2x\right)^{6-8x}\left(\frac{72-32x}{(3-2x)^2} + \left(-8\ln(3-2x) - \frac{2(6-8x)}{3-2x}\right)^2\right) > 0
	\end{align*}
	It follows that $h_2(x) > h_1(x)$ for all $0 \le x \le \frac{1}{512}$.
	Therefore, we get:
	\begin{align*}
	\Delta &> (1-2a)^{7.5}\left(5-\frac{1}{128}\right)^{5-\frac{1}{64}} - 4(-5120a + 3^6) \\ &> (1-15a)(4\cdot 3^6 + 100) - 4(-5120a + 3^6) \\ &> 0
	\end{align*}
	where the second inequality follows by Bernoulli's inequality and the last follows by the assumption $a \le \frac{1}{512}$.	
\end{proof}

\begin{lem}
	\label{lem:ab_identity1}
	Let $0 < a,b < 1$ such that $b \ge 2a$ and $a+2b < 1$. Then $$2H(a+b) - 3H(a) + H(2a) - H(a+2b) > 0$$
\end{lem}
\begin{proof}
	Fix $a < 1$ and define $g(b) = 2H(a+b) - 3H(a) + H(2a) - H(a+2b)$. By Lemma \ref{lem:a_identity} we have $g(2a) = 2H(3a) - 3H(a) + H(2a) - H(5a) > 0$. It therefore suffices to show that $g'(b) > 0$ for all $2a \le b < \frac{1-a}{2}$. We have,
	\begin{align*}
	g'(b) &= -2\log\left(\frac{a+b}{1-(a+b)}\right) + 2\log\left(\frac{a+2b}{1-(a+2b)}\right) \\ &= 2\log\left(\frac{(1-(a+b))(a+2b)}{(a+b)(1-(a+2b))}\right) \\ &> 0
	\end{align*}
	
	where the inequality follows since
	\begin{align*}
	(1-(a+b))(a+2b) &= a+2b - (a+b)(a+2b) \\ &> a+b - (a+b)(a+2b) \\ &= (a+b)(1-(a+2b))
	\end{align*}
\end{proof}

\begin{lem}
	\label{lem:ab_identity2}
	Let $0 < a,b < 1$ such that $a+4b < 1$. Then $$H(a+2b) - H(a+b) -\frac{1}{4}H(a+4b) + \frac{1}{4}H(a) > 0$$ 
\end{lem}
\begin{proof}
	Fix $a < 1$ and define $g(b) = H(a+2b) - H(a+b) -\frac{1}{4}H(a+4b) + \frac{1}{4}H(a)$. We have $$g(0) = H(a) - H(a) -\frac{1}{4}H(a) + \frac{1}{4}H(a) = 0$$
	and therefore it suffices to show that $g'(b) > 0$ for all $b < \frac{1-a}{4}$. We have,
	\begin{align*}
	g'(b) &= -2\log\left(\frac{a+2b}{1-(a+2b)}\right) +  \log\left(\frac{a+b}{1-(a+b)}\right) + \log\left(\frac{a+4b}{1-(a+4b)}\right) \\ &= \log\left(\frac{(1-(a+2b))^2(a+b)(a+4b) }{(a+2b)^2(1-(a+b))(1-(a+4b))}\right)
	\end{align*}
	We first notice that $$(a+b)(a+4b) = a^2 + 5ab+4b^2 > a^2 + 4ab + 4b^2 = (a+2b)^2$$
	Thus, to finish the proof, it suffices to show that 
	\begin{align*}
	\Delta \triangleq (1-(a+2b))^2 - (1-(a+b))(1-(a+4b)) > 0
	\end{align*} 
	which is true because
	\begin{align*}
	\Delta &= 1 - 2a -4b + a^2 + 4ab + 4b^2 - (1-2a-5b + a^2 + 5ab + 4b^2) \\ &= b - ab > 0
	\end{align*}
\end{proof}

\begin{lem}
	\label{lem:ab_identity3}
	Let $0 < a \le \frac{1}{8}$ and $b = \frac{1}{4}(1-a)$. Then $$\Delta = H(a+2b) - H(a+b) + \frac{1}{4}H(a)> 0$$
\end{lem}
\begin{proof}
	Since $H\left(\frac{1}{2}a+\frac{1}{2}\right) = H\left(\frac{1}{2} - \frac{1}{2}a\right)$ and $\frac{1}{2} - \frac{1}{2}a \ge \frac{3}{4}a+\frac{1}{4}$ for $a \le \frac{1}{8}$, it follows that $$H\left(\frac{1}{2}a+\frac{1}{2}\right) \ge H\left(\frac{3}{4}a+\frac{1}{4}\right)$$
	Therefore,  
	$$\Delta = H\left(\frac{1}{2}a+\frac{1}{2}\right) - H\left(\frac{3}{4}a+\frac{1}{4}\right) + \frac{1}{4}H(a) \ge \frac{1}{4}H(a) > 0$$ 
\end{proof}

\begin{lem}
	\label{lem:numerical}
	Let $x = \frac{1}{2^k}$ and $y = \frac{1}{2^m}$ where $m$ and $l$ are integers such that $1 \le k < m$. Then the following three inequalities hold.
	\begin{enumerate}
		\item $ \frac{1}{2}H(x+y-xy) - \frac{3}{4}H(y) + \frac{1}{4}H(2y) - \frac{1}{4}H\left(2x+y-2xy\right) > 0$.
		\item If $k \ge 3$ then $H(2x+y-2xy) + \frac{1}{4}H(y) - H(x+y-xy) - \frac{1}{4}H(4x+y-4xy) > 0$.
		\item If $k=2$ then $H(2x+y-2xy) + \frac{1}{4}H(y) - H(x+y-xy) > 0$.
		\item $H(2y) - H(y) - \frac{1}{4}H(4y) > 0$.
	\end{enumerate}	  
\end{lem}
\begin{proof}
	\begin{enumerate}
		\item For $k < m-1$ it holds that $x-xy \ge 2y$. Therefore, in this case the identity follows by plugging $a=y$ and $b=x-xy$ in Lemma \ref{lem:ab_identity1}. If $k=m-1$ then $x-xy = 2y-2y^2$ and the identity follows by Lemma \ref{lem:a_identity2}.
		\item This follows by plugging $a=y$ and $b=x-xy$ in Lemma \ref{lem:ab_identity2}.
		\item This follow by Lemma \ref{lem:ab_identity3} with $a=y$.
		\item This follows by plugging $a=0$ and $b=y$ in Lemma \ref{lem:ab_identity2}. 
	\end{enumerate}
\end{proof}

We now proceed to show that $T_t = B_t$.  We need to show that \topdown\,first builds the right-path of the tree and then builds all left sub-trees of nodes in the right-path from top to bottom. The proof has three parts:

\underline{Part 1:} Let $v_1$ and $v_2$ be the right child and left child of the root, respectively. We first show that $W(v_1) \ge W(v_2)$. Recall that by the results in \citep{fiat2004decision}, \topdown\,chooses an $x$-variable as its root (this is WLOG if $l=m$).

Denote $x = \frac{1}{2^{l-1}}$ and $y = \frac{1}{2^m}$. Then,
\begin{align*}
W(v_1) = \frac{1}{2}H\left(x+y-xy\right) - \frac{1}{4}H\left(y\right) - \frac{1}{4}H\left(2x+y-2xy\right) \numberthis \label{eq:left_child}
\end{align*}
and
\begin{align*}
W(v_2)=\frac{1}{2}H(1-y) - \frac{1}{4}H(1-2y) \numberthis \label{eq:right_child}
\end{align*}

Therefore, $$W(v_1) - W(v_2) = \frac{1}{2}H(x+y-xy) - \frac{3}{4}H(y) + \frac{1}{4}H(2y) - \frac{1}{4}H\left(2x+y-2xy\right)$$
which is positive by part 1 of Lemma \ref{lem:numerical}.

\underline{Part 2:} Next, let $v$ be a node on the right path and let $u$ be its parent. We will show that $W(v) \ge W(u)$. First assume that $l\ge 3$. Let $1 \le k \le l-2$ be the depth of $u$, $x = \frac{1}{2^{l-k}}$ and $y=\frac{1}{2^m}$. Then
\begin{align*}
2^k W(v) = \frac{1}{2}H(2x+y-2xy) - \frac{1}{4}H(y) - \frac{1}{4}H(4x+y-4xy)
\end{align*}
and
\begin{align*}
2^k W(u) = H(x+y-xy) - \frac{1}{2}H(y) - \frac{1}{2}H(2x+y-2xy)
\end{align*}

Therefore:

\begin{align*}
2^k W(v) - 2^k W(u) &= H(2x+y-2xy) + \frac{1}{4}H(y) - H(x+y-xy)  -\frac{1}{4}H(4x+y-4xy) \\ &> 0
\end{align*}

by part 2 of Lemma \ref{lem:numerical}. Therefore, $W(v) > W(u)$. 

If $l=2$ then $$2^k W(v) = \frac{1}{2}H(2x+y-2xy) - \frac{1}{4}H(y)$$ and therefore,
\begin{align*}
2^k W(v) - 2^k W(u) &= H(2x+y-2xy) + \frac{1}{4}H(y) - H(x+y-xy)  \\ &> 0
\end{align*}
by part 3 of Lemma \ref{lem:numerical}.

\underline{Part 3:}  Let $v$ be a non-root node of a left sub-tree of a node in the right-path. Let $u$ be its parent. Let $y = \frac{1}{2^{m-k}}$ where $m \ge 2$ and $ 0 \le k \le m-2$ is the depth of $u$ in the left sub-tree. We have:

\begin{align*}
2^{d(u)}\left(W(v) - W(u)\right) &= \frac{1}{2}H(2y) - \frac{1}{4}H(4y) - \left(H(y) - \frac{1}{2}H(2y)\right) \\ &= H(2y) - H(y) - \frac{1}{4}H(4y) > 0
\end{align*}
where the inequality follows by part 4 of Lemma \ref{lem:numerical}. Therefore, $W(v) > W(u)$.

\underline{Finishing the proof:} By part 1 and part 2, \topdown\,will first grow the right path of the tree before expanding any left sub-tree. To see this, let $v_1$ and $v_2$ be two left children of nodes in the right-path and $v_1$ is in a higher left sub-tree. Note that $v_1$ and $v_2$ are $y$-nodes. Then, it holds that $W(v_1) > W(v_2)$. Thus, by parts 1 and 2 the weighted gain of any node on the right path is larger than the weighted gain of any left child of a node in the right-path.

Next, after completing the right path, \topdown\,will expand the left child of the root (and not other roots of left-sub trees by the previous argument). By part 3, it will then expand the highest left sub-tree. By the same argument, it will continue to expand each left sub-tree from top to bottom.

\subsection{Proof of Proposition \ref{prop:bt_optimal}}

We first need to prove several lemmas.
\begin{lem}
	\label{lem:optimal_tree_basic}
	For any DNF with at most 2 terms, there exists an optimal tree that consists only of $x$-nodes or $y$-nodes. Furthermore, 
	\begin{enumerate}
		\item In the case of one term, the tree is in $\mc_1$.
		\item In the case of 2 terms, the tree is in $\mc_2$. 
	\end{enumerate}
\end{lem}
\begin{proof}
	We prove the claim by induction on $t+s$ where $t$ is the size of the tree and $s$ is the number of literals in the DNF. For $t=s=1$ the claim holds. Assume it holds for $t+s$. First assume that the DNF has one term with $y$-variables. Consider the root of an optimal tree $T$. If it is not in the DNF, then both $T_R$ and $T_L$ correspond to the ground-truth DNF and are of size less than $t$. By induction, they are WLOG right-paths that consist of only $y$-nodes. Assume WLOG that $E(T_R) \le E(T_L)$. Then $E(T) = \frac{1}{2}E(T_L) + \frac{1}{2}E(T_R) \ge E(T_R)$, where the equality follows by the fact that the DNF of the root of $T_L$ and the DNF of the root of $T_R$ are equal to the ground-truth DNF, and that the variables are independent. Thus by replacing $T$ with $T_R$ we get a new tree in $\mc_1$ with less than $t$ nodes and did not increase the test error. We can add more nodes to get a tree of size $t$ and not increase the test error.
	
	If the root node is in the DNF then its left sub-tree is WLOG a leaf, because a single leaf is optimal for a constant function. By induction, its right sub-tree is WLOG a right-path, and therefore the whole tree is in $\mc_1$. 
	
	Now, assume that the DNF has two terms. Consider the root of an optimal tree. If it is not in the DNF, then both its right and left sub-tree correspond to the ground-truth DNF and are of size less than $t$. By induction they are both in $\mc_2$ and consist only of $x$-nodes and $y$-nodes. Assume WLOG that $E(T_R) \le E(T_L)$. Then $E(T) = \frac{1}{2}E(T_L) + \frac{1}{2}E(T_R) \ge E(T_R)$, where the equality follows by the fact that the DNFs of $T_L$ and $T_R$ are the ground-truth DNF. Thus by replacing $T$ with $T_R$ we get a new tree in $\mc_2$ with less than $t$ nodes and did not increase the test error. We can add more nodes to get a tree of size $t$ and not increase the test error.
	
	If the root is in the DNF, by induction the left sub-tree is in $\mc_1$ and the right sub-tree in $\mc_2$. Therefore, the tree is in $\mc_2$. Furthermore, all nodes are $x$-nodes or $y$-nodes by induction.
\end{proof}

\begin{lem}
	\label{lem:cv_calc}
	Let $v$ be an $x$-node with $F(v) = (x_1 \wedge x_2 \wedge \cdots \wedge x_{k_1}) \vee (y_1 \wedge y_2 \wedge \cdots \wedge y_{k_2})$. Then the following holds: \footnote{Analogous claims hold when $v$ is a $y$-node. We also use the notation that for $k_2=0$ we get a one-term DNF with $x$-variables.}
	\begin{enumerate}
		\item If $k_1 = 1$, $k_2 = 0$, then $C(v) = \frac{1}{2}$.
		\item If $k_2 \ge 2$, $k_1 = 2$, then $C(v) = \frac{1}{2^{k_2}} - \frac{1}{2^{k_2+1}}$.
		\item If $k_2 \ge 2$, $k_1 = 1$, then $C(v) = \frac{1}{2} - \frac{1}{2^{k_2}}$.
		\item If $k_1 > 1$ and $k_2 = 0$ or $k_1,k_2 > 2$, then $C(v) = 0$.
	\end{enumerate}  
\end{lem}
\begin{proof}
	By direct calculation we get:
	\begin{enumerate}
		\item $C(v) = C\left(\frac{1}{2}\right) - \frac{1}{2}C(1) - \frac{1}{2}C(0) = \frac{1}{2}$.
		\item \begin{align*}C(v) &= C\left(\frac{1}{2^{k_1}} + \frac{1}{2^{k_2}} - \frac{1}{2^{k_1+k_2}}\right) - \frac{1}{2}C\left(\frac{1}{2^{k_1-1}} + \frac{1}{2^{k_2}} - \frac{1}{2^{k_1+k_2-1}}\right) -\frac{1}{2}C\left(\frac{1}{2^{k_2}}\right) \\ &= \frac{1}{4} + \frac{1}{2^{k_2}} - \frac{1}{2^{2+k_2}}	- \frac{1}{2^{k_2+2}} + \frac{1}{2^{k_2+1}} - \frac{1}{4} - \frac{1}{2^{k_2+1}}	\\ &= \frac{1}{2^{k_2}} - \frac{1}{2^{k_2+1}}		\end{align*}
		\item \begin{align*}C(v) &= C\left(\frac{1}{2^{k_1}} + \frac{1}{2^{k_2}} - \frac{1}{2^{k_1+k_2}}\right)  - \frac{1}{2}C(1)  -\frac{1}{2}C\left(\frac{1}{2^{k_2}}\right) \\ &= 
		\frac{1}{2^{k_2+1}}
		+ \frac{1}{2} - \frac{1}{2^{k_2}} -  \frac{1}{2^{k_2+1}}	\\ &= \frac{1}{2} - \frac{1}{2^{k_2}}	\end{align*}
		\item In the first case,
		$C(v) = C\left(\frac{1}{2^{k_1}}\right) - \frac{1}{2}C\left(\frac{1}{2^{k_1-1}}\right) - \frac{1}{2}C(0) = 0$ 
		
		In the second case, \begin{align*}C(v) &= C\left(\frac{1}{2^{k_1}} + \frac{1}{2^{k_2}} - \frac{1}{2^{k_1+k_2}}\right) - \frac{1}{2}C\left(\frac{1}{2^{k_1-1}} + \frac{1}{2^{k_2}} - \frac{1}{2^{k_1+k_2-1}}\right) -\frac{1}{2}C\left(\frac{1}{2^{k_2}}\right) \\ &= \frac{1}{2^{k_1}} + \frac{1}{2^{k_2}} - \frac{1}{2^{k_1+k_2}} -
		\frac{1}{2^{k_1}} - \frac{1}{2^{k_2+1}} + \frac{1}{2^{k_1+k_2}}
		-\frac{1}{2^{k_2+1}}
		\\ &= 0
		\end{align*}
	\end{enumerate}
\end{proof}

\begin{lem}
	\label{lem:right_xnode_tree}
	Assume that $T \in \mc_2$ has $t$ nodes, has optimal test error among all trees in $\mt_t$ and its right path consists only of $x$-nodes. Then $C(T) = C(B_t)$.
\end{lem}
\begin{proof}
	If $t \le l$, then the optimal tree is a right-path with $x$-nodes. This follows since by Lemma \ref{lem:cv_calc}, the only nodes with $C(v) > 0$ in a tree of size less than $l$ are the $l-1$ and $l$ nodes in the right-path.
	
	Assume that $t > l$. By Lemma \ref{lem:optimal_tree_basic}, WLOG each left sub-tree of an $x$-node in the right-path is a right $y$-path. Assume by contradiction that $T$ has an $x$-right-path with less than $l$ nodes. Then by Lemma \ref{lem:cv_calc}, its error reduction is at most
	\begin{equation}
	\label{eq:full_right_xpath}
	\sum_{i=1}^{l-1}{\frac{1}{2^{m+i}}} +  \frac{1}{2^{m+l-1}} = \frac{1}{2^m} - \frac{1}{2^{m+l-1}} + \frac{1}{2^{m+l-1}} = \frac{1}{2^{m}} \le \frac{1}{2^l}
	\end{equation}
	
	where the first summand on the left side is the total test error reduction of all left sub-trees which are right $y$-paths. The second summand is the error reduction of the $l-1$ node in the right $x$-path of the tree. The right-side of \Eqref{eq:full_right_xpath} is the error reduction of a full right $x$-path. Indeed, by Lemma \ref{lem:cv_calc} this is given by the sum of error reduction of nodes $l-1$ and $l$ in the path, which is $\frac{1}{2^{m+l-1}} + \frac{1}{2^l} - \frac{1}{2^{m+l-1}} = \frac{1}{2^l}$. Therefore, a full right $x$-path has error reduction at least as any other tree with no full right $x$-path. We can thus assume WLOG that the optimal tree has a full right $x$-path.
	
	We are left to show that there is no node in a left sub-tree (which is a right $y$-path), unless all left sub-trees above it are full right $y$-paths. Assume by contradiction that this does not hold and consider a left sub-tree $T_L'$ which violates this condition. We claim that by moving the nodes of $T_L'$ to higher left sub-trees can only increase the error reduction. To see this, first consider the case where $T_L'$ is not full. In this case the error reduction of all its nodes is 0. We can therefore remove them without decreasing error reduction. Placing them in higher left sub-trees can only increase the error reduction. 
	
	If $T_L'$ is a full right $y$-path, let $T_L''$ be a higher left sub-tree which is not full (it exists by our assumption). By moving nodes from $T_L'$ to $T_L''$ we can make $T_L''$ full. By part 1 of Lemma \ref{lem:cv_calc}, the error reduction of the whole tree increases by $\frac{1}{2^{m+k_2}} - \frac{1}{2^{m+k_1}} > 0$ where $k_1$ and $k_2$ are the depths of the last nodes in $T_L'$ and $T_L''$, respectively.       
\end{proof}

We now turn to proving the proposition. We prove it by induction on $t$. For $t=1$, \topdown\,chooses an $x$-node which is $B_1$. Assume the claim holds for $t-1$. Let $T$ be the optimal tree with $t$ nodes. We have 4 cases:
\underline{Case 1: The root of $T$ is an $x$-node.}  By the induction hypothesis, there exists an optimal tree $T'$ whose right-most path consists only of $x$-nodes. This follows since we can replace the right sub-tree of the root of $T$ with a tree $B_{t'}$, where $t'$ is the number of inner nodes in the right sub-tree, to get the tree $T'$. By induction, this does not increase the test error. Therefore, by Lemma \ref{lem:right_xnode_tree}, $B_t$ is optimal.

\underline{Case 2: The root of $T$ is a $y$-node and $l=m$.} Similarly to the previous case, we get by the induction hypothesis that there exists an optimal tree $T'$ whose right path consists only of $y$-nodes. By Lemma \ref{lem:right_xnode_tree}, applied to the case where $l=m$ and right $y$-path (which by symmetry can be analyzed the same as the case of a right $x$-path), $B_t$ is optimal.

\underline{Case 3: The root of $T$ is a $y$-node, $l<m$ and the left sub-tree of the root is a full $x$-tree.} The following procedure does not decrease the test error reduction and does not increase the size of the tree: Remove root and left sub-tree and add one $y$-node to each left sub-tree of an $x$-node in the right path. Note that there are at most $l$ $x$-nodes in the right path. Add remaining nodes as $x$-nodes to right path. We will denote by $T'$ the new resulting tree.

First, we note that by adding a $y$-node to each left sub-tree of an $x$-node, the total error reduction of each such left-sub tree does not change. To see this, consider two cases. In the first the left sub-tree of an $x$-node is not full. Then its error reduction is 0 and therefore, removing the root cannot decrease its error reduction. Second, if the left sub-tree is full, then the only node with non-zero error reduction is the last node. Its error reduction is $\frac{1}{2^{d(v)+1}}$. After removing the root and adding a $y$-node, the last node still has error reduction of $\frac{1}{2^{d(v)+1}}$. 

By Lemma \ref{lem:cv_calc} part 1, the error reduction of the left sub-tree of $T$ is $\frac{1}{2^{l+1}}$. Denote by $D$ the difference between the error reduction of the right-path of $T'$ and the error reduction of the right-path of $T$. It suffices to show that $D \ge \frac{1}{2^{l+1}}$. Let $C_r$ be the error reduction of the right-path of $T$. Since $T$ has a root $y$-node, we get by Lemma \ref{lem:cv_calc} that $C_r \le \frac{1}{2^{l+1}}$. By the induction hypothesis and the assumption that the left sub-tree of $r_T$ is full, $T'$ has $l$ $x$-nodes on its right-path. Therefore, by Lemma \ref{lem:cv_calc}, we have  $D = \frac{1}{2^l} - C_r \ge \frac{1}{2^{l+1}}$, as desired.

\underline{Case 4: The root of $T$ is a $y$-node, $l<m$ and the left sub-tree of the root is \textit{not} a full $x$-tree.}

Assume that there exists an $x$-node in the right-most path whose left-sub tree is a full right $y$-path. Then, we can create a new tree by moving nodes from the latter left sub-tree to the left sub-tree of the root, to create a full right $x$-path in the left sub-tree of the root. This does not decrease the test error reduction. Then we are in the previous case where the left sub-tree of the root is a full $x$-path.

We are left with the case in which there is no $x$-node in the right-most path with full left-sub tree. In this case, we can remove the root and its left sub-tree and replace the tree with its right sub-tree. If there is no right sub-tree then we can replace the tree with a single $x$-node. The total test error reduction will not decrease. To see this, note that there is no error reduction in left sub-trees of $x$-nodes (because they are not full) and the error reduction in the left sub-trees of $y$-nodes can only increase because we increase the probability to get to each node. 

Finally, the error reduction of the right-path can only increase as well. Note that by induction, WLOG the nodes on the right-path, not including the root, are $x$-nodes. If the length of the right-path is less than $l$, then every node has zero error reduction and this will not change after removal of the root and its left sub-tree. If its length is $l$, then the only non-zero error reduction is at the last node in the path, denoted by $v$. This error reduction is $\underbrace{\frac{1}{2^{l-1}}}_{\text{p(v)}}\underbrace{\frac{1}{2^{m}}}_{\text{C(v)}} = \frac{1}{2^{l+m-1}}$, by Lemma \ref{lem:cv_calc} part 2. After the removal of nodes, its error reduction is $\underbrace{\frac{1}{2^{l-2}}}_{\text{p(v)}}\underbrace{\frac{1}{2^{m+1}}}_{\text{C(v)}} = \frac{1}{2^{l+m-1}}$, and therefore does not change. If its length is $l+1$, i.e.,it has $l$ $x$-nodes, then the total error reduction before removal is $$\frac{1}{2^{l-1}}\frac{1}{2^{m}} + \frac{1}{2^{l}}\left(\frac{1}{2} - \frac{1}{2^{m-1}}\right) = \frac{1}{2^{l+1}}$$
and after
$$\frac{1}{2^{l-2}}\frac{1}{2^{m+1}} + \frac{1}{2^{l-1}}\left(\frac{1}{2} - \frac{1}{2^{m}}\right) = \frac{1}{2^{l}}$$
and thus the error reduction increases, which finishes the proof. 

\section{Comparison with BestFirst}

\begin{figure}[t]
	\begin{algorithm}[H]
		\caption{$\text{BestFirst}_{\md}(t)$}
		\label{alg:bestfirst}
		\begin{algorithmic}
			\STATE Initialize $T$ to be a single leaf labeled by the majority label with respect to $\md$.
			\STATE \textbf{while} $T$ has less than $t$ nodes: 
			\begin{ALC@g}
				\STATE $\Delta_{best} \leftarrow 0$.
				\STATE \textbf{for} each pair $l \in \ell(T)$ and $i \in F_l$:
				\begin{ALC@g}
					\STATE $\Delta \leftarrow H(q(l)) -(1-\tau_i) H(q(l_0)) - \tau_i H(q(l_1))$.
					\STATE \textbf{if} $\Delta \ge \Delta_{best}$ \textbf{then}:
					\begin{ALC@g}
						\STATE $\Delta_{best} \leftarrow \Delta$; $l_{best} \leftarrow l$; $i_{best} \leftarrow i$.
					\end{ALC@g}
				\end{ALC@g}
				\STATE $T \leftarrow T(l_{best},i_{best})$.
			\end{ALC@g}
			\STATE \textbf{return} $T$
		\end{algorithmic}
	\end{algorithm}
	\caption{BestFirst algorithm.}
	\label{fig:bestfirst}
\end{figure}

Here we show that BestFirst chooses the tree in Figure \ref{fig:best_first_tree}. Figure \ref{fig:bestfirst} shows the pseudo-code for BestFirst, where we use the same notation as in Section \ref{sec:preliminaries}. By parts 1 and 2 at the end of the proof of Proposition \ref{prop:topdown_bt}, BestFirst will first grow the right-path of $x$-nodes. Note that in part 2 we show that a node on the right-path has larger weighted gain than its parent. Since it is a deeper node it also has a larger gain. However, the gain of the left childs of the nodes on the right-path are equal. Therefore, it can choose to grow the left child of the second $x$-node on the right-path. Then, by part 3 at the end of the proof of Proposition \ref{prop:topdown_bt}, it will expand the left sub-tree of the second $x$-node before expanding the left child of the first $x$-node. Here again, we use the fact that if the weighted gain of a child is larger than the weighted gain of its parent, then also its gain is larger than the gain of its parent.

\section{Correctness of Dynamic Programming Formulas}

First, we will show by induction on $t$ that there is an optimal tree which all of its nodes are variables in the DNF. Therefore, we do not need to consider variables not in the DNF in the calculation of optimal test error. In the induction step, if by contradiction there is an optimal tree $T$ with a root node that is not in the DNF, we can replace it with one of its sub-trees $T_R$ or $T_L$ (whose nodes are all in the DNF by the induction hypothesis) without increasing the test error. This follows since the test error of $T$ is a convex combination of the test errors of $T_R$ and $T_L$, by the independence assumption.

We initialize $OPT(F,0)$ to be the test error of a single root node with DNF $F$ and $OPT(\emptyset,t) = 0$. Furthermore, in \eqref{eq:product_exp_formula} we set $S_i(F,t,j,c) = 1$ if $c$ does not have a variable with distribution $Bernoulli(p_i)$. In both formulas of uniform and bernoulli distributions, the dynamic programming algorithm chooses the feature and number of nodes for left and right sub-trees that result in the lowest test error. The optimal test error of the tree is calculated by the sum over both sub-trees of the probability to reach the sub-tree times the optimal test error of the sub-tree. The optimal sub-tree of size $j$ is the tree with best test error among all trees of size $j$ with ground-truth DNF which is the DNF of the root after conditioning on the variable in the root. Thus, if this root variable is 0, the term which contains the root variable is removed from the root DNF. Otherwise if it is 1, then the root variable is removed from the DNF.

Now we proceed to prove the correctness of the formulas. Consider an optimal tree $T$ over $\md$ with a ground truth DNF $F$ and consider WLOG its left sub-tree $T'$ with $t$ internal nodes and the DNF $F'$ which is the DNF $F$ after conditioning on the assignments from the root of $T$ to the root of $T'$ (which is the left child of the root of $T$). Since $\md$ is a product distribution and by optimality of $T$, $T'$ is the optimal tree with $t$ internal nodes over $\md$ with a ground-truth DNF $F'$ (i.e., the same marginal distribution on $\mx$ but realizable with $F'$ and not $F$). Furthermore, the test error of $T$ is the lowest among all trees with different splitting variables at the root and all possible number of internal nodes for the left and right sub-trees.

\section{Finite Sample Experiments Details}

We implemented both TopDown and BestFirst for finite samples. We considered both uniform and product distirbutions with ground DNF with $n=38$ and 10 terms of sizes $[3,3,4,5,3,4,3,5,4,4]$. We experimented with product distributions as defined in Section \ref{sec:product_exps} with $(p_1,p_2) \in \{(0.4,0.6), (0.3,0.7)\}$. We chose at random half of the variables to be $Bernoulli(p_1)$ and the other $Bernoulli(p_2)$. We experimented with training set sizes $|S|$ of $[50,100,1000,2000,3000,4000,10000, 20000,50000]$ and test set size 10000. For each training set size we performed 10 experiments with a different sampled training set.

	
	

\end{document}